\documentclass[sn-mathphys-num]{sn-jnl}

\usepackage{graphicx}%
\usepackage{multirow}%
\usepackage{amsmath,amssymb,amsfonts}%
\usepackage{amsthm}%
\usepackage[title]{appendix}%
\usepackage{xcolor}%
\usepackage{textcomp}%
\usepackage{manyfoot}%
\usepackage{booktabs}%
\usepackage{algorithm}%
\usepackage{algorithmicx}%
\usepackage{algpseudocode}%
\usepackage{listings}%

\usepackage{float}
\usepackage{graphicx}
\usepackage[caption=false]{subfig}
\usepackage{lmodern}

\newtheorem{theorem}{Theorem}
\newtheorem{proposition}[theorem]{Proposition}%
\newtheorem{lemma}{Lemma}
\newtheorem{corollary}{Corollary}
\newtheorem{definition}{Definition}%

\begin{document}

\title[Article Title]{Error Bounds of Supervised Classification from Information-Theoretic Perspective}


\author*[1]{\fnm{Binchuan} \sur{Qi}}\email{2080068@tongji.edu.cn}



\affil[1]{\orgdiv{College of Electronics and Information Engineering}, \orgname{Tongji University}, \orgaddress{\city{Shanghai}, \postcode{201804}, \country{China}}}

\abstract{
In this paper, we explore bounds on the expected risk when using deep neural networks for supervised classification from an information theoretic perspective.
Firstly, we introduce model risk and fitting error, which are derived from further decomposing the empirical risk. Model risk represents the expected value of the loss under the model's predicted probabilities and is exclusively dependent on the model. Fitting error measures the disparity between the empirical risk and model risk. 
Then, we derive the upper bound on fitting error, which links the back-propagated gradient and the model's parameter count with the fitting error. 
Furthermore, we demonstrate that the generalization errors are bounded by the classification uncertainty, which is characterized by both the smoothness of the distribution and the sample size. 
Based on the bounds on fitting error and generalization, by utilizing the triangle inequality, we establish an upper bound on the expected risk. 
This bound is applied to provide theoretical explanations for overparameterization, non-convex optimization and flat minima in deep learning.
Finally, empirical verification confirms a significant positive correlation between the derived theoretical bounds and the practical expected risk, thereby affirming the practical relevance of the theoretical findings.}


\keywords{deep learning, generalization, Kullback-Leibler (KL) divergence, information theory.}

\maketitle
\section{Introduction}\label{sec:intro}
Deep learning (DL)'s remarkable practical achievements have a profound impact on the conceptual foundations of machine learning and artificial intelligence.
Unraveling the factors that contribute to the performance of deep neural networks (DNNs) would not only enhance model interpretability but also facilitate the development of more principled and trustworthy architectural designs.
However, explaining why DNNs can generalize well despite their overwhelming capacity remains an open research challenge.
In the statistical learning framework, the objective of learning is to minimize the expected risk, which is composed of the empirical risk and the generalization error~\cite{valiant1984theory,Bousquet2002StabilityAG,mukherjee2006learning}.
A model is considered effective if both the generalization error and empirical risk are minimized.
The empirical risk refers to the error or loss of a model when evaluated on the training dataset.
The generalization error measures the population risk, representing the average loss over new, randomly sampled data from the population. 
Extensive experimental evidence indicates that the generalization of DNNs significantly influences prediction performance~\cite{Kawaguchi2017GeneralizationID}, alongside their expressivity~\cite{Leshno1993OriginalCM, Barron1993UniversalAB} and trainability~\cite{Choromaska2014TheLS, Kawaguchi2016DeepLW}.
Some classical theory work attributes the generalization ability to the employment of a low-complexity class of hypotheses and has proposed various complexity measures, including the
Vapnik-Chervonenkis (VC) dimension~\cite{Wu2021StatisticalLT}, Rademacher complexity~\cite{Bartlett2003RademacherAG}, and uniform stability~\cite{mukherjee2006learning, Bousquet2002StabilityAG}, to control generalization error.
However, recent works like~\citet{Zhang2016UnderstandingDL},~\citet{Zhang2021UnderstandingDL},~\citet{He2015DeepRL}, and~\citet{Belkin2018ReconcilingMM} reveal that high-capacity DNNs can still generalize effectively with genuine labels, yet generalize poorly when trained with random labels. 
This finding contradicts the traditional wisdom in statistical learning theory, which posits that complex models tend to overfit the training data~\cite{vapnik1971chervonenkis,Wu2021StatisticalLT,Bartlett2003RademacherAG}. In addition to the complexity of the hypothesis space, research~\cite{Keskar2016OnLT,Neyshabur2014InSO} highlights the impact of optimization methods, training data, and neural network architectures on the generalization of deep models. 

In recent years, information-theoretic methods, which jointly consider the hypothesis class, learning algorithm, and the data distribution, have drawn considerable attention since the works of~\citet{Russo2015ControllingBI} and~\citet{Xu2017InformationtheoreticAO}.
The advantage of applying information-theoretic tools in understanding generalization lies in the fact that generalization is influenced by various aspects, including the data distribution, the complexity of the hypothesis class, and the characteristics of the optimization algorithm. For instance, the mutual information (MI) between the training sample and the trained parameter weights has been applied in various generalization bounds~\cite{Raginsky2016InformationtheoreticAO, Pensia2018GeneralizationEB, Asadi2018ChainingMI,Goldfeld2018EstimatingIF}. 

Existing error bounds, grounded in information theory, primarily concentrate on bounding generalization errors, with limited examination of empirical and expected risks. 
To the best of our knowledge, these bounds fall short in explaining over-parameterization and non-convex optimization mechanisms in DNNs. 
In addition, upper bounds based on mutual information encounter significant computational challenges and are often impractical for real-world application.

Consequently, this paper delves into the error bounds of supervised classification, a pivotal topic within the domain of machine learning.
Classical statistical learning theory decomposes expected risk into empirical risk and generalization error, focusing primarily on bounding the latter.
In contrast, our methodology introduces the notions of fitting error and model risk, which stem from further decomposing the empirical risk. Model risk signifies the expected value of the loss under the model's predicted probabilities and is exclusively contingent upon the model. Fitting error gauges the disparity between the empirical risk and model risk.
Subsequently, we derive an upper bound on fitting error, which associates the back-propagated gradient and the model's parameter count with the fitting error.
Moreover, we illustrate that generalization errors can be effectively modulated by the classification uncertainty, which is shaped by both the support size of the data distribution and the sample size. With a sufficiently large sample size, we establish a correlation between the classification uncertainty and the smoothness of the data distribution, implying that the classification uncertainty can act as an indicator of the training dataset's reliability and guide the setting of regularization hyperparameters.
Finally, employing the triangle inequality and grounded in the bounds on fitting error and generalization, we establish an upper bound on the expected risk.
This bound highlights the impact of the classification uncertainty and the fitting errors on the expected risk, and under mild conditions, it yields several important insights: 
\begin{itemize}
    \item  Increasing the number of training samples and decreasing the maximum loss value contribute to a reduction in the expected risk.
    \item The expected risk can be effectively controlled by the gradient of the Kullback-Leibler (KL) divergence between the empirical conditional probability distribution and the predictive probability distribution of the model. 
    \item Gradient-based algorithms can lead to a lower expected risk as the number of parameters grows.
    \item A smaller maximum eigenvalue of the structural matrix, $\lambda_{\max}(S(\theta,x))$, which is proved to be equivalent to the flat minima, is advantageous in minimizing the expected risk.
\end{itemize}

The article's structure is organized as follows. 
In Section~\ref{sec:related_work}, we delve into the existing literature and discuss relevant research. 
In Section\ref{sec:pre}, we begin by defining the fundamental concepts and objects employed throughout the paper to establish a clear foundation. 
In Section~\ref{sec:fitting_error}, an upper bound on fitting error is derived. 
In Section~\ref{sec:complexity}, we introduce the concept of classification uncertainty and prove that generalization errors can be bounded by measure, emphasizing its significance in learning. 
In Section~\ref{sec:expected_risk}, we derive an upper bound on expected risk and analyse the learning techniques for DNNs based on this bound.
In Section~\ref{sec:application}, based on the upper bounds derived earlier, we analyze and discuss the mechanisms of overparameterization, non-convex optimization and flat minima. 
In Section~\ref{sec:experiments}, empirical evidence is presented, demonstrating a strong positive correlation between the derived upper bounds and the expected risk in real-world scenarios, thereby validating our theoretical findings.
In Section~\ref{sec:conclusion}, the article concludes with a summary of our key contributions and highlights some important implications. 

\section{Related Work} 
\label{sec:related_work}
Our work is intrinsically connected to the extensive literature on information-theoretic generalization bounds, which have been empirically demonstrated their efficacy in capturing the generalization behavior of DNNs~\citep{Negrea2019InformationTheoreticGB,Hellstrm2022ANF}.
The original information-theoretic bound introduced by~\cite{Xu2017InformationtheoreticAO} has been extended or improved in various ways, such as the chaining method~\citep{Asadi2018ChainingMI,Zhou2022StochasticCA,Clerico2022ChainedGB}, the random subset and individual technique~\cite{Negrea2019InformationTheoreticGB,Bu2019TighteningMI,Haghifam2020SharpenedGB,Galvez2021TighterEG,Zhou2020IndividuallyCI}. 

Information-theoretic generalization bounds have also been explored for some specific algorithms. For example, the weight-based information-theoretic bounds have been successfully applied to characterize the generalization properties of the stochastic gradient Langevin dynamics (SGLD) algorithm~\cite{Pensia2018GeneralizationEB,Bu2019TighteningMI,Negrea2019InformationTheoreticGB,Haghifam2020SharpenedGB,Galvez2021TighterEG,Wang2021AnalyzingTG}, and more recently, these bounds have also been used to analyze either vanilla Stochastic Gradient Descent (SGD)~\cite{Neu2021InformationTheoreticGB,Wang2021OnTG} or the stochastic differential equations (SDEs) approximation of SGD~\cite{Wang2021OnTG}.
The information bottleneck principle~\cite{ShwartzZiv2017OpeningTB,Ahuja2021InvariancePM,Wongso2023UsingSM},
 from the perspective of the information change in neural networks, posits that for better prediction, the neural network learns useful information and filters out redundant information. Beyond supervised learning, information-theoretic bounds are also useful in a variety of learning scenarios, including meta-learning~\cite{Hellstrm2022EvaluatedCB}, semi-supervised learning~\cite{He2021InformationTheoreticCO}, transfer learning~\cite{Wu2020InformationtheoreticAF}, among others. 

Another popular claim in the literature is that flat minima generalize better than sharp (non-flat) minima~\cite{Hochreiter1997FlatM,Keskar2016OnLT}. 
In response to this notion,~\citet{Foret2020SharpnessAwareMF} proposed Sharpness-Aware Minimization (SAM), a method designed to directly optimize for flat solutions. Flatness is often quantified mathematically with the largest eigenvalue of the Hessian matrix of the loss with respect to the model parameters on the training set~\cite{Jastrzebski2017ThreeFI,Lewkowycz2020TheLL,Dinh2017SharpMC}.

A parallel related line of research is the over-parameterization theory of DNNs 
~\cite{Du2018GradientDP, Zou2018StochasticGD, AllenZhu2018ACT,Sirignano2018MeanFA,Mei2018AMF, Chizat2018OnTG}. Despite the inherent non-convexity of the objective functions, empirical evidence indicates that gradient-based methods, such as stochastic gradient descent (SGD), are capable of converging to the global minima in these networks. Key insights from~\citet{Du2018GradientDP,Chizat2018OnLT,Arjevani2022AnnihilationOS} emphasize the pivotal role of the over-parameterization~\citep{Yun2018SmallNI} in enhancing the generalization ability of DNNs. 
The neural tangent kernel (NTK) thus has emerged as a pivotal concept, as it captures the dynamics of over-parameterized neural network trained by gradient descent~\cite{Jacot2018NeuralTK}. 
Notably, recent studies~\citep{Du2018GradientDP, Zou2018StochasticGD, AllenZhu2018ACT} demonstrate that, with specialized scaling and random initialization, the dynamics of gradient descent (GD) in continuous-width multi-layer neural networks can be tracked via NTK and the convergence to the global optimum with a linear rate can be proven. 

\section{Preliminaries}
\label{sec:pre}
In this section, we provide essential background information on learning algorithms and fundamental concepts, particularly focusing on generalization error. 
Our analysis on supervised classification is conducted within the standard statistical learning framework. 
The instance space, denoted as $\mathcal Z:=\mathcal{X}\times \mathcal{Y}$, consists of the input feature space $\mathcal{X}$ and the label set $\mathcal{Y}$.
The hypothesis space, a set of functions parameterized by $\Theta$, is represented by $\mathcal F_{\Theta}=\{f_\theta:\theta\in \Theta\}$, where $\Theta$ denotes the parameter space.
Training data is represented by an $n$-tuple $S^n =\{(x^{(i)},y^{(i)})\}_{i=1}^n$, which consists of independent and identically distributed (i.i.d.) samples drawn from an unknown true distribution $\bar q$. 
The empirical probability mass function (PMF) of $(X,Y)$ derived from these samples $S^n$, denoted by $q^n$ (or simply $q$), is defined as $q(x,y):=\frac 1 n \sum_{i=1}^n\mathbf{1}_{\{x,y\}}(x^{(i)},y^{(i)})$. 
Obviously, we can obtain $nq\sim \operatorname{Multinomial}(n, \bar{q})$.
According to the (weak) law of large numbers, $q$ converges to $\bar q$ in probability. 
We simplify the notation by defining $q_{Y|x}(y):=q_{Y|X}(y|x),q_{Y|x}=( q_{Y|x}(y_1)$, and $\cdots,q_{Y|x}(y_{|\mathcal Y|}))^\top$, where $q_{X}(x)$ and $q_{Y|X}(y|x)$ represent the marginal and conditional PMF, respectively, with the one-hot vector $y_i\in \mathcal Y$ for each $i$.
For a hypothesis $f_\theta\in \mathcal{F}$ (abbreviation: $f$), the $j$-th element of $f_\theta(x)$ is denoted by $[f_\theta(x)]_j$. 
To convert this into a probability distribution, we apply the softmax function, resulting in the prediction probabilities for $x$ as follows: 
\begin{equation}\label{eqx1}
p_{Y|x}=(\frac{e^{f_{\theta}(x)_1}}{Z_\theta(x)},\cdots,\frac{e^{f_{\theta}(x)_{|\mathcal{Y}|}}}{Z_\theta(x)})^\top,
\end{equation}
where $Z_\theta(x)$, known as the normalizing constant (a.k.a. the partition function), is given by:
\begin{equation}\label{eqZ}
    Z_\theta(x) = \sum_{j=1}^{|\mathcal{Y}|} e^{f_{\theta}(x)_j}.
\end{equation}
In this paper, we refer to $p_{Y|x}$ as the model's predicted probability. 

For the sake of conciseness and uniformity in description, we introduce the function $R_{\ell}(f_\theta,q):\mathcal{F}_\Theta(x)\times \Delta^{|\mathcal{Y}|}\to \mathbb R$ to denote the expected value of the loss $\ell(f_\theta(x),y)$ under the distribution $q$, which is expressed as follows:
\begin{equation}
 R_{\ell}(f_\theta,q)=\mathbb E_{(X,Y)\sim q}\ell(f_\theta(x),y).
\end{equation} 
Therefore, utilizing the function $R_{\ell}(f_\theta,q)$, the expected risk, empirical risk and generalization error of a hypothesis $ f_\theta $ under a given loss function $ \ell: \mathcal{F}_\Theta(x) \times \mathcal{Y} \to \mathbb{R} $ are expressed as follows:
\begin{equation}\label{eq_Rf}
\begin{aligned}
    \mathcal{R}(f_\theta)&=\mathbb E_{(X,Y)\sim \bar q}\ell(f_\theta(x),y)\\
    &=R_{\ell}(f_\theta,\bar q),\\
        \mathcal{R}_S(f_\theta)&=\frac{1}{n}\sum_{i=1}^n \ell(f_\theta(x^{(i)}),y^{(i)})\\
    &=\mathbb E_{(X,Y)\sim  q}\ell(f_\theta(x),y)\\
    &=R_{\ell}(f_\theta,q),\\
        \mathrm{gen}(\ell,f_\theta)&=|R_{\ell}(f_\theta,\bar q)-R_{\ell}(f_\theta,q)|.
\end{aligned}
\end{equation}
Below, we will define the concepts of fitting error and model risk introduced in this paper.
\begin{definition}[Model risk]
    The model risk, denoted as $\mathcal{R}_M(f_\theta)$, is defined by 
\begin{equation}
\begin{aligned}
    \mathcal{R}_M(f_\theta)&=\mathbb E_{(X,Y)\sim  p}\ell(f_\theta(x),y)\\
    &=R_{\ell}(f_\theta,p),
\end{aligned}
\end{equation}where $p(x,y) = q_X(x)p_{Y|x}(y)$. 
\end{definition}
\begin{definition}[Fitting error]
    we introduce the fitting error, to measure the gap between empirical risk and model risk: \begin{equation}
\begin{aligned}
    \mathrm{fit}(\ell,f_\theta):=|R_{\ell}(f_\theta,p)-R_{\ell}(f_\theta,q)|.
\end{aligned}
\end{equation}
\end{definition}
The fitting error essentially reflects the discrepancy between the distribution produced by the model and the empirical PMF. 
To focus on the effect of the characteristics of the model, we define the normalized fitting error as follows.
\begin{definition}[normalized fitting error]
    The normalized fitting error is denoted by: 
\begin{equation}
    {\mathrm{fit_n}(\ell,f_\theta)}:={\mathrm{fit}(\ell,f_\theta)}/{\sqrt{\mathbb E_X[\|\ell(f_\theta(x))\|^2_2]}},
\end{equation}
where $\ell(f_\theta(x))=(\ell(f_\theta(x),y_1),\cdots,\ell(f_\theta(x),y_{|\mathcal{Y}|}))^\top$.
\end{definition}
To clarify the relationships between expected risk, empirical risk, model risk, and generalization error, we will utilize the function $R_{\ell}(f_\theta,\cdot)$ to represent these risks and errors in the subsequent analysis. 
To characterize the structural properties of the model, we introduce the following definition: 
\begin{definition}[Structural matrix]\label{def_e_A}
We define $S(\theta,x) = \nabla_\theta f_\theta(x)^\top \nabla_\theta f_\theta(x)$ as the structural matrix.
\end{definition}
Besides, we introduce several key terms used to derive the bound of fitting error as follows:
\begin{equation}\begin{aligned}
&\ell(f_\theta(x))=(\ell(f_\theta(x),y_1),\cdots,\ell(f_\theta(x),y_{|\mathcal{Y}|}))^\top,\\
&F(\theta,q_{Y|x})=\frac{\|\nabla_\theta D_{KL}(q_{Y|x}\|p_{Y|x})\|_2^2}{\|\nabla_\theta f_\theta(x)\|_F^2},\\
&F(\theta_j,q_{Y|x})=\frac{\|\nabla_{\theta_j} D_{KL}(q_{Y|x}\|p_{Y|x})\|_2^2}{\|\nabla_\theta f_\theta(x)\|_2^2},\\
&G(\theta,q_{Y|x})=\frac{\sum_{i=1}^m \|(q_{Y|x}-p_{Y|x})\times \nabla_{\theta_i} f_\theta(x)\|_2^2}{\|\nabla_\theta f_\theta(x)\|_F^2},\\
&G(\theta_j,q_{Y|x})=\frac{\|(q_{Y|x}-p_{Y|x})\times \nabla_{\theta_j} f_\theta(x)\|_2^2}{\|\nabla_{\theta_j} f_\theta(x)\|_2^2},
\end{aligned}
\end{equation}
where $m$ represents the number of parameters in the model, and $\theta_j$ denotes the $j$-th parameter of the model.

The theorems proved in the subsequent sections will rely on the following lemma. 
\begin{lemma}[Markov inequality]
\label{lemma_markov}
Let $Z$ be a non-negative random variable. Then, for all $t \geq 0$, $\operatorname{Pr}(Z \geq t) \leq \mathbb{E}[Z] / t$. 
\end{lemma}

\section{Bound on Fitting Error}
\label{sec:fitting_error}
In this section, we aim to bound the fitting error and examine the influence of the model's architecture on it. 
First, we present two lemmas, essential for establishing an upper bound on the fitting error.
\begin{lemma}
\label{lemma_kexiinequality}
Given $X,Y$ are two random variables, with finite second moments, i.e., $\mathbb E[X^2]<\infty,\mathbb E[Y^2]<\infty$, the following inequality holds: 
\begin{equation}
(\mathbb E[\langle X,Y\rangle])^2\le \mathbb E[\|X\|^2]\mathbb E[\|Y\|^2].
\end{equation}
\end{lemma}
\begin{proof}
Consider the function $f(\alpha)=\mathbb E[\|\alpha X-Y\|^2]$, which upon expansion yields the quadratic form:
\begin{equation}f(\alpha)=\alpha^2 ~\mathbb E[\|X\|^2] - 2~ \alpha~ \mathbb E[~|\langle X,Y\rangle|~] + \mathbb E[\|Y\|^2] \ge 0.\end{equation}
The discriminant of this quadratic function $\Delta=b^2-4ac$ is given by:
\begin{equation}\Delta=b^2-4ac=\mathbb E[~2~ |\langle X,Y\rangle|~ ]^2 -4~ \mathbb E[\|X\|^2]~ \mathbb E[\|Y\|^2].\end{equation} 
Because $f$ has at most one root, the discriminant must be non-positive, i.e., $\Delta\le 0$, and the equality holds iif $\alpha X=Y$.
\end{proof}
\begin{lemma}
\label{lemma_equality}
$\|Ax\|_2^2=\|A\|^2_F\|x\|_2^2-\sum_{k=1}^m \|a_k\times x\|_2^2$, where $A$ is a $m \times n$ matrix, $x\in \mathbb R^n$, $\|\cdot\|_2$ is the Euclidean norm, $\|\cdot\|_F$ is the Frobenius norm, and $\|a_k\times x\|_2^2=\frac{1}{2}\sum_{i \mathop = 1}^n\sum_{j\mathop=1}^n (a_{ki} x_j - a_{kj} x_i)^2$.
\end{lemma}
\begin{proof}
        \label{appendix:proof_fitting_error}
Because the following equation always holds
\begin{equation}\begin{aligned}
\Big(\sum_{j=1}^n a_{kj} x_j\Big)^2 = \sum_{j=1}^n a_{kj}^2 \sum_{j=1}^n x_j^2-\frac 1 2 \sum_{i \mathop = 1}^n \sum_{j=1}^n  (a_{ki} x_j - a_{kj} x_i)^2\end{aligned},\end{equation} 
 we obtain 
\begin{equation}\begin{aligned}
\sum_{k=1}^m&\Big(\sum_{j=1}^n a_{kj} x_j\Big)^2 = \\
&\sum_{k=1}^m \sum_{j=1}^n a_{kj}^2 \sum_{j=1}^n x_j^2- \sum_{k=1}^m \|a_k\times x\|_2^2.
\end{aligned}\end{equation}
For 
\begin{equation}\|Ax\|^2_2 = \sum_{k=1}^m (\sum_{j=1}^n a_{kj}x_j )^2,\end{equation}
then we obtain
\begin{equation}\|Ax\|_2^2=\|A\|^2_F\|x\|_2^2-\sum_{k=1}^m \|a_k\times x\|_2^2.\end{equation}
\end{proof}

The subsequent theorem provides an upper bound on the fitting error.
\begin{theorem}[Bound on fitting error]
\label{prop:fitting_error}
\begin{equation}
\begin{aligned}
\mathrm{fit_n}(\ell,f_\theta)&\le \sqrt{\mathbb E_X[F(\theta,q_{Y|x})]+\mathbb E_X[G(\theta,q_{Y|x})]} \\
&=\sqrt{\mathbb E_X[F(\theta_j,q_{Y|x})]+\mathbb E_X[G(\theta_j,q_{Y|x})]},
\end{aligned}
\end{equation}
where $X\sim q_X$.
\end{theorem}
\begin{proof}
By inserting  
\begin{equation}p_{Y|x}(y)=\frac{ e^{f_\theta(x)}}{Z_\theta(x)},\end{equation} into $D_{KL}(q_{Y|x}\|p_{Y|x})$, we obtain
\begin{equation}D_{KL}(q_{Y|x}\|p_{Y|x})=q_{Y|x}^\top f_\theta(x)- \log Z_\theta(x)-H(q_{Y|x}).\end{equation}
Taking the gradient of the function $D_{KL}(q_{Y|x}\|p_{Y|x})$ with respect to $\theta$, we have: \begin{equation}\nabla_\theta D_{KL}(q_{Y|x}\|p_{Y|x})=\nabla_\theta f_\theta(x)^\top (q_{Y|x}-p_{Y|x}).\end{equation}
Applying Lemma~\ref{lemma_equality}, we derive: 
\begin{equation}\begin{aligned}
\|\nabla_\theta &D_{KL}(q_{Y|x}\|p_{Y|x})\|_2^2= \|q_{Y|x}-p_{Y|x}\|_2^2\|\nabla_\theta f_\theta(x)\|_F^2\\
&-\sum_{i=1}^{|\mathcal{Y}|} \|(q_{Y|x}-p_{Y|x})\times \nabla_{\theta_i} f_\theta(x)\|_2^2,
\end{aligned}\end{equation}
where $m$ is the number of parameters.
Extracting $\|q_{Y|x}-p_{Y|x}\|_2^2$ from this expression, we get:
\begin{equation}\label{eq12x}
\begin{aligned}
\|q_{Y|x}-&p_{Y|x}\|_2^2 = \frac{\|\nabla_\theta D_{KL}(q_{Y|x}\|p_{Y|x})\|_2^2}{\|\nabla_\theta f_\theta(x)\|_F^2}\\
&+\frac{\sum_{i=1}^m \|(q_{Y|x}-p_{Y|x})\times \nabla_{\theta_i} f_\theta(x)\|_2^2}{\|\nabla_\theta f_\theta(x)\|_F^2}. 
\end{aligned}
\end{equation}
Repeating the process for the derivative with respect to a single parameter $\theta_j$, we obtain a similar conclusion: 
\begin{equation}
\begin{aligned}
\|q_{Y|x}-p_{Y|x}\|_2^2 &=\frac{\|\nabla_{\theta_j} D_{KL}(q_{Y|x}\|p_{Y|x})\|_2^2}{\|\nabla_{\theta_j} f_\theta(x)\|_F^2}\\
&+\frac{\|(q_{Y|x}-p_{Y|x})\times \nabla_{\theta_j} f_\theta(x)\|_2^2}{\|\nabla_{\theta_j} f_\theta(x)\|_2^2},
\end{aligned}
\end{equation}
According to the definitions of $\ell(f_\theta(x))$, $F(\theta,q_{Y|x})$, $F(\theta_j,q_{Y|x})$, $G(\theta,q_{Y|x})$ and $G(\theta_j,q_{Y|x})$, we have 
\begin{equation}\label{eq12}
\begin{aligned}
\|q_{Y|x}-p_{Y|x}\|_2^2 &= F(\theta,q_{Y|x})+G(\theta,q_{Y|x})\\
&= F(\theta_j,q_{Y|x})+G(\theta_j,q_{Y|x}), 
\end{aligned}
\end{equation}
where $j\in [1,m]$, $m$ is the number of parameters.

Using Lemma~\ref{lemma_kexiinequality} and Equation~\eqref{eq12}, we can bound the expected absolute difference as follows:
\begin{equation}\label{eq19}
\begin{aligned}
\mathbb E_X &|p_{Y|x}^\top \ell(f_\theta(x))-q_{Y|x}^\top \ell(f_\theta(x))|\\
&\le \sqrt{\mathbb E_X[\|p_{Y|x}-q_{Y|x}\|^2_2] \mathbb E_x[\|\ell(f_\theta(x))\|^2_2]}\\
&=\sqrt{\mathbb E_X[F(\theta,q_{Y|x})+G(\theta,q_{Y|x})] \mathbb E_x[\|\ell(f_\theta(x))\|^2_2]}\\
&=\sqrt{\mathbb E_X[F(\theta_j,q_{Y|x})+G(\theta_j,q_{Y|x})] \mathbb E_x[\|\ell(f_\theta(x))\|^2_2]}.
\end{aligned}
\end{equation}

Since $p(x,y)=q_X(x)p_{Y|x}(y)$, we have: 
\begin{equation}\begin{aligned}
\mathbb E_{(X,Y)\sim p} \ell(f_\theta(x),y)&=\frac{1}{n}\sum_{i=1}^n\ell([f_\theta(x)]_i,y_i)\\
&=\mathbb E_{X\sim q_X} p_{Y|x}^\top \ell(f_\theta(x)).
\end{aligned}\end{equation}
This leads to: 
\begin{equation}\label{eq20}
\begin{aligned}
&|\mathbb E_{(X,Y)\sim p} \ell(f_\theta(x),y)-\mathbb E_{(X,Y)\sim q} \ell(f_\theta(x),y)|\\
&\le \sqrt{\mathbb E_X[F(\theta,q_{Y|x})+G(\theta,q_{Y|x})] \mathbb E_X[\|\ell(f_\theta(x))\|^2_2]}\\
&=\sqrt{\mathbb E_X[F(\theta_j,q_{Y|x})+G(\theta_j,q_{Y|x})]\mathbb E_X[\|\ell(f_\theta(x))\|^2_2] }.
\end{aligned}
\end{equation}
Divide both sides of the inequality by $\sqrt{\mathbb E_X[\|\ell(f_\theta(x))\|^2_2]}$, we obtain 
\begin{equation}
\begin{aligned}
\frac{\mathrm{fit}(\ell,f_\theta)}{\sqrt{\mathbb E_X[\|\ell(f_\theta(x))\|^2_2]}}&\le \sqrt{\mathbb E_X[F(\theta,q_{Y|x})+G(\theta,q_{Y|x})]} \\
&=\sqrt{\mathbb E_X[F(\theta_j,q_{Y|x})+G(\theta_j,q_{Y|x})]}.
\end{aligned}
\end{equation}
\end{proof}

\section{Bound on Generalization Error} 
\label{sec:complexity}
This section aims to establish the connection between generalization error and the classification uncertainty, and explore the application of classification uncertainty in practical classification scenarios.
Some classical theoretical studies attribute the generalization ability to the employment of a low-complexity hypothesis and have introduced various complexity measures for hypotheses to control generalization error. 
The source of $\mathrm{gen}(\ell,f_\theta)$ is caused by the insufficiency of training samples, which fail to adequately represent the distribution of the entire dataset. Therefore, we offer a new perspective by exploring the concept of generalization error from a data-centered perspective.
Initially, we establish the following lemmas, which are essential to derive the upper bound on generalization error.
\begin{lemma}\label{lemma_KL}
Given any distributions $q,\bar q$ on the space $\mathcal X$ and a function $f:\mathcal X\to [0,L]$, the following inequality holds:
$D(q\|\bar q)\ge \frac{2}{L^2}(\mathbb E_q f-\mathbb E_{\bar q} f)^2$.
\end{lemma}
\begin{proof}
    \label{appendix:proof_lemma_kl}
    Begin with the following calculation: 
\begin{equation}\label{eq008}
\begin{aligned}
|\mathbb E_q & f-\mathbb E_{\bar q} f|=|\sum_{x\in \mathcal X} (q(x)-\bar q(x))f(x)| \\
&=|\sum_{x\in\mathcal X}(\bar q(x)-q(x))(f(x)-\frac{L}{2})+\frac{L}{2}(\bar q(x)-q(x))|\\
&\le \sum_{x\in\mathcal X}|\bar q(x)-q(x)||f(x)-\frac{L}{2}|\\
&\le \frac{L}{2}\|q-\bar q\|_1 .   
\end{aligned}
\end{equation}
    Pinsker's inequality~\citep{1967Information,Kullback1967ALB} states that: 
\begin{equation}\label{eq009}
    D(q\|\bar q)\ge \frac{1}{2}\|q-\bar q\|_1^{2}.
\end{equation} 
Combine the two inequalities~\eqref{eq008} and~\eqref{eq009}, we obtain: 
\begin{equation}
D(q\|\bar q)\ge \frac{2}{L^2}|\mathbb E_q f-\mathbb E_{\bar q} f|^2.
\end{equation}
\end{proof}

To establish bound on the generalization error, we introduce the following theorem.
\begin{theorem}[Bound on generalization error]
    \label{prop:generalization_error}
\begin{equation}\Pr(\mathrm{gen}(\ell,f_\theta)\ge \epsilon )\le \frac{L^2\mathbb E_{\bar q|q} D_{KL}(q\|\bar q)}{2\epsilon^2}\end{equation}
where  $L=\|\ell(y,f_\theta(x))\|_\infty$, $nq\sim \operatorname{Multinomial}(n, \bar{q})$. 
\end{theorem}
\begin{proof}
    Since $q$ is the empirical PMF of $(X,Y)$ based on the training set $S^n$, Lemma~\ref{lemma_KL} implies that $D_{KL}(q\|\bar q)\ge \frac{2}{L^2}\mathrm{gen}(\ell,f_\theta)^2$, where $L=\|\ell(y,f_\theta(x))\|_\infty$. This leads to:
\begin{equation}
\begin{aligned}
    \Pr(D_{KL}(q\|\bar q)\ge t) 
    &\ge \Pr(\frac{2}{L^2}\mathrm{gen}(\ell,f_\theta)^2\ge t)\\
    &= \Pr(\mathrm{gen}(\ell,f_\theta)^2\ge L^2t/2)\\
    &=\Pr(\mathrm{gen}(\ell,f_\theta) \ge L\sqrt{t/2}).
\end{aligned}
\end{equation}
Let $\epsilon=L\sqrt{t/2}$, because $\mathrm{gen}(\ell,f_\theta), D_{KL}(q\|\bar q)$ are functions of the random variable $\bar q$, applying Lemma~\ref{lemma_markov}, we obtain: 
\begin{equation}
\begin{aligned}
 \Pr(\mathrm{gen}(\ell,f_\theta)\ge \epsilon )&\le \Pr(D_{KL}(q\|\bar q) \ge \frac{2\epsilon^2}{L^2}) \\
 &\le \frac{L^2\mathbb E_{\bar q|q} D_{KL}(q\|\bar q)}{2\epsilon^2}.
\end{aligned}
\end{equation}
\end{proof}
As shown in Theorem~\ref{prop:generalization_error}, the generalization error is influenced not only by the maximum loss $\ell(\ell,f_\theta)$ but also by $\mathbb E_{\bar q|q} D_{KL}(q\|\bar q)$. 
Consequently, we define the classification uncertainty as follows.
\begin{definition}[Classification uncertainty]
The classification uncertainty is a measure of the generalization error inherent to the characteristics of the learning task, given by
\begin{equation}
    C(q)=\frac{\mathbb E_{\bar q|q} D_{KL}(q\|\bar q)}{2},
\end{equation}where $nq\sim \operatorname{Multinomial}(n, \bar{q})$.    
\end{definition}
To estimate the classification uncertainty efficiently, we present the following propositions.
\begin{proposition}
\label{prop:complexity_bound}
    \begin{enumerate}
    \item As the sampling size escalates, the classification uncertainty tends asymptotically towards zero: 
\begin{equation}
    \lim_{n\to \infty}C(q)=0.
\end{equation}
    \item If the sample data is indifferent and all existing constraints are satisfied, a more uniform $q$ implies less classification uncertainty. 
    \end{enumerate}
\end{proposition}
\begin{proof}
    \label{appendix:proof_complexity_bound}
According to the Bayesian theorem and Sanov's theorem~\cite{2003Elements}, we have 
\begin{eqnarray}
    \begin{aligned}
        &\Pr(\bar q|q)={\Pr(q|\bar q)\Pr(\bar q)}/{\Pr(q)}\\
        &\lim_{n\to \infty} -\frac{1}{n}\log \Pr(q|\bar q)= D_{KL}(q\|\bar q).
    \end{aligned}
\end{eqnarray}
Combining the above equations, we obtain the following limits as $n$ approaches infinity: 
\begin{equation}\label{eq:appto_kl}
\begin{aligned}
    \lim_{n\to \infty}\Pr(\bar q=q|q)&= 1\\
\lim_{n\to \infty}\mathbb E_{\bar q|q}D_{KL}(q\|\bar q)&= 0.
\end{aligned}
\end{equation}

First, let's find the lower bound of $C(q)$:  
\begin{equation}
    \begin{aligned}
        C(q)&= -H(q)- \mathbb E_{\bar q|q}\mathbb E_{Z\sim q} \log \bar q(z)\\
        &= -H(q)- \mathbb E_{Z\sim q} \mathbb E_{\bar q|q} \log \bar q(z)\\
        &\ge -H(q)- \mathbb E_{Z\sim q}  \log \mathbb E_{\bar q|q} \bar q(z),
    \end{aligned}
\end{equation}
where the final inequality is derived from Jensen's inequality, as $-\log(\cdot)$ is a convex function.
Given that $\sum_{Z} \mathbb E_{\bar q|q}\bar q(z)=1$, therefore we have
\begin{equation}\label{eq:lb_unc}
    C(q)\ge D_{KL}(q\|\mathbb E_{\bar q|q}\bar q).
\end{equation}

Then the upper bound of $C(q)$ is derived as follows: 
\begin{equation}\label{eq:ub_unc}
    \begin{aligned}
        C(q)&= -H(q)- \mathbb E_{\bar q|q}\mathbb E_{Z\sim q} \log \bar q(z)\\
        &=-H(q)+\mathbb E_{\bar q|q}\mathbb E_{Z\sim q} [1+\frac{1-\bar q(z)}{\bar q(z)}]\\
        &\le -H(q)+\mathbb E_{\bar q|q}\mathbb E_{Z\sim q} \frac{1-\bar q(z)}{\bar q(z)},
    \end{aligned}
\end{equation}
where the final inequality is based on $\log (1+x)\le x$. 
Combining the formula~\eqref{eq:lb_unc} and~\eqref{eq:ub_unc}, we have
\begin{equation}
    \begin{aligned}
            D_{KL}(q\|\mathbb E_{\bar q|q}\bar q)&\le C(q)\\
            &\le -H(q)+\mathbb E_{\bar q|q}\mathbb E_{Z\sim q} \frac{1-\bar q(z)}{\bar q(z)}\\
            &\le H(q)+\mathbb E_{\bar q|q} \sum_Z \frac{q_{\max}}{\bar q(z)}-1.
    \end{aligned}
\end{equation}
Because $\bar q$ is an arbitrary distribution on $\mathcal{Z}$, when $n$ is insufficiently large, the conditional probability $\Pr(\bar q=q|q)$ fails to approximate to 1, causing $\mathbb E_{\bar q|q}\bar q$ to be relatively uniform. Consequently, a more uniform $q$ often implies a smaller lower $D_{KL}(q\|\mathbb E_{\bar q|q}\bar q)$, i.e., a smaller bound of $C(q)$.
The fact that $q_{\max}$ influences the upper bound also highlights the advantage of a uniform $q$ in minimizing $C(q)$.
Thus, under the condition of satisfying existing constraints, a more uniform $q$ implies lower classification uncertainty.
\end{proof}
The first conclusion of Proposition~\ref{prop:complexity_bound} is consistent with the law of large numbers: more samples lead to reduced bias from random sampling. 
The second conclusion indicates that, given the assumption of i.i.d. sampling, the upper limit of classification uncertainty is directly related to the smoothness of $q$ and the sample size. 
The well-known principle of maximum entropy posits that the most probable distribution subject to specified constraints should be the maximum entropy distribution. Furthermore, $H(q)$ serves as an indicator of the uniformity of $q$, thus establishing a degree of equivalence between the negative Shannon entropy $-H(q)$ and the classification uncertainty $C(q)$. 

As posited by Theorem~\ref{prop:generalization_error}, the classification uncertainty, $C(q)$, serves as an indicator of the discrepancy between the training dataset and the entire population, reflecting the reliability of the training data. 
Consequently, we harness this classification uncertainty to inform the selection of hyperparameters for the regularization term. 
Intuitively, an increase in classification uncertainty necessitates a heightened focus on the regularization loss. 
Furthermore, when the model perfectly aligns with the training data $S^n$ (i.e., $p=q$), the classification uncertainty can serve as an indicator of confidence in the model's predictions. 
In summary, we demonstrate that the generalization error can be controlled by the classification uncertainty, and conclude that the number of training samples as well as the uniformity of the distribution jointly determine the classification uncertainty.

\section{Bound on Expected Risk}
\label{sec:expected_risk}
In this section, we leverage the conclusions of generalization error and fitting error discussed in the previous section to establish an upper bound on the expected risk, offering a new perspective on classification. 
By applying the triangle inequality of absolute values, we formulate the following theorem.
\begin{theorem}[Bound on expected risk]
    \label{prop:expected_risk}
\begin{equation}\begin{aligned}
R_{\ell}(f_\theta,\bar q)&\le \mathrm{gen}(\ell,f_\theta)+R_{\ell}(f_\theta,p)\\
&+\sqrt{\mathbb E_X[F(\theta,q_{Y|x})+G(\theta,q_{Y|x})] \mathbb E_x[\|\ell(f_\theta(x))\|^2_2]},
\end{aligned}\end{equation}
where  
\begin{equation}
\begin{aligned}
&\Pr(\mathrm{gen}(\ell,f_\theta)\ge \epsilon )\le \frac{L^2C(q)}{2\epsilon^2},\\
&L=\|\ell(y,f_\theta(x))\|_\infty,\\
&nq\sim \operatorname{Multinomial}(n, \bar{q}),\\
&F(\theta,q_{Y|x})=\frac{\|\nabla_\theta D_{KL}(q_{Y|x}\|p_{Y|x})\|_2^2}{\|\nabla_\theta f_\theta(x)\|_F^2},\\
&G(\theta,q_{Y|x})=\frac{\sum_{i=1}^m \|(q_{Y|x}-p_{Y|x})\times \nabla_{\theta_i} f_\theta(x)\|_2^2}{\|\nabla_\theta f_\theta(x)\|_F^2}.
\end{aligned}
\end{equation} 
\end{theorem}
\begin{proof}
Utilizing the triangle inequality for absolute value, we obtain
\begin{equation}\begin{aligned}
|R_{\ell}(f_\theta,\bar q)-R_{\ell}(f_\theta,p)|
&\le |R_{\ell}(f_\theta,\bar q)-R_{\ell}(f_\theta,q)|\\
&+|R_{\ell}(f_\theta,q)-R_{\ell}(f_\theta,p)| \\
&= \mathrm{gen}(\ell,f_\theta)+\mathrm{fit}(\ell,f_\theta).
\end{aligned}
\end{equation}
Eliminating the absolute value signs yields the inequality:
\begin{equation}R_{\ell}(f_\theta,\bar q)\le R_{\ell}(f_\theta,p)+\mathrm{gen}(\ell,f_\theta)+\mathrm{fit}(\ell,f_\theta).\end{equation}
By applying Theorem~\ref{appendix:proof_complexity_bound} and~\ref{appendix:proof_fitting_error} to this inequality, we can prove the theorem.
\end{proof}
Unlike traditional statistical learning frameworks that use empirical risk and generalization error to bound the expected risk, our theorem employs model risk, generalization error, and fitting error to control the expected risk. 
Contrary to the conventional generalization theories~\cite{Wu2021StatisticalLT,Bartlett2003RademacherAG} that bound the generalization error by the complexity of the hypothesis space, we utilize the classification uncertainty for bounding purposes. 
We further derive an upper bound for $F(\theta,q_{Y|x})$, which will be utilized in the analysis of flat minima and in experimental verification of our conclusions.
\begin{proposition}
\label{prop:bound_F}
Let $\lambda_{\max}(S(\theta,x))$ be the largest eigenvalue of structural matrix. We obtain the following bound: 
\begin{equation}F(\theta,q_{Y|x})\le \frac{\lambda_{\max}(S(\theta,x))\|q_{Y|x}-p_{Y|x}\|_2^2}{\mathrm{trace}(S(\theta,x))}.\end{equation}
\end{proposition}
\begin{proof}
    Since 
\begin{equation}\|\nabla_\theta D_{KL}(q_{Y|x}\|p_{Y|x})\|_2^2=\|\nabla_\theta f_\theta(x)^\top (q_{Y|x}-p_{Y|x})\|_2^2,\end{equation}
this equation leads to the inequality: 
\begin{equation}\begin{aligned}
\|\nabla_\theta D_{KL}(q_{Y|x}\|p_{Y|x})\|_2^2
&\le \lambda_{\max}(S(\theta,x)) \|(q_{Y|x}-p_{Y|x})\|_2^2,
\end{aligned}\end{equation} where $\lambda_{\max}(S(\theta,x))$ is the largest eigenvalue of structural matrix.
According to the definition of $F(\theta,q_{Y|x})$, we obtain the following bound: 
\begin{equation}F(\theta,q_{Y|x})\le \frac{\lambda_{\max}(S(\theta,x))\|q_{Y|x}-p_{Y|x}\|_2^2}{\mathrm{trace}(S(\theta,x))}.\end{equation}
\end{proof}
Drawing from the findings in Section~\ref{sec:complexity} and Section~\ref{sec:fitting_error}, we make the following inferences about the training techniques for DNNs.
\begin{corollary}
\label{cor:expecetd_error}
    \begin{enumerate}
        \item         When the classification uncertainty $C(q)$ of the training dataset is higher, the weight of the regularization term should also be higher. 
        \item Lowering the model risk $\mathcal{R}_M(f_\theta)$ is beneficial for reducing expected risk. When cross-entropy is selected as the loss function, we have $\mathcal{R}_M(f_\theta)=H_p(Y|X)$. When keeping $H_p(Y)$ constant, reducing model risk is equivalent to increasing the mutual information between $X$ and $Y$ where $(X,Y)\sim p$.
        \item Reducing the largest eigenvalue of $S(\theta,x)$ is advantageous for controlling fitting errors and expected risk. 
    \end{enumerate}
\end{corollary}
\begin{proof}
    Corollary 1 are direct consequences of Theorem~\ref{prop:expected_risk}.
    For a cross-entropy loss, $\mathcal{R}_M(f_\theta)$ simplifies to the negative conditional entropy of the random variables $X$ and $Y$ when $(X,Y)\sim p$:
\begin{equation}
    \mathcal{R}_M(f_\theta)=-\mathbb E_{(X,Y)\sim  p} \log p_{Y|x}=H_p(Y|X),
\end{equation}where $H_p(Y|X)$ represents the conditional entropy of $Y$ given $X$.
    Since $H_p(Y|X)=H_p(Y)-I_p(X;Y)$, thus when keeping $H_p(Y)$ constant, decreasing $H_p(Y|X)$ is equivalent to increasing $I_p(X;Y)$. Corollary 2 is then proved.
    According to Proposition~\ref{prop:bound_F}, it follows that
 minimizing $\lambda_{\max}(S(\theta,x))$ helps to minimizing $F(\theta,q_{Y|x})$. 
\end{proof}

\section{Application to Deep Learning}
\label{sec:application}
In this section, based on the upper bounds derived earlier, we analyze and discuss the mechanisms of overparameterization, non-convex optimization and flat minima. 
\subsection{Overparameterization and Fitting Error}
\label{subsec:op}
In this subsection, from our derived upper bound on normalized fitting error, we infer theoretically that an increase in the number of parameters corresponds to the decrease of fitting error, which theoretically demonstrates the impact of model parameter quantity on learning performance.. 
According to Theorem~\ref{prop:fitting_error}, the bound on fitting error comprises two components: $\mathbb E_{X}G(\theta,q_{Y|x})$ and  $\mathbb E_{X} F(\theta,q_{Y|x})$.
Based on the definition of $G(\theta_i,q_{Y|x})$, we can express it as: 
\begin{equation}G(\theta_i,q_{Y|x})=\|q_{Y|x}-p_{Y|x}\|^2_2\sin ^2\alpha(\theta_i,x),\end{equation}
where $\alpha(\theta_i,x)$ denotes the angle between vectors $(q_{Y|x}-p_{Y|x})$ and $\nabla_{\theta_i}f_\theta(x)$.
Gradient-based algorithms do not directly influence $\alpha(\theta_i,x)$ during the optimization of the empirical risk. 
Therefore, assuming $\alpha(\theta_i,x)$ behaves as a random variable after training is reasonable.
Considering a model with $m$ parameters, $\{\alpha(\theta_1,x),\cdots,\alpha(\theta_m,x)\}$ can be regarded as $m$ random samples from a distribution.
We define $G_M(\theta_{1}^{m},q_{Y|x})$ as the minimum of $\{\mathbb E_X G(\theta_1,q_{Y|x}),\cdots,\mathbb E_X G(\theta_m,q_{Y|x})\}$:  
\begin{equation}G_M(\theta_{1}^{m},q_{Y|x})=\min\{\mathbb E_X G(\theta_1,q_{Y|x}),\cdots,\mathbb E_X G(\theta_m,q_{Y|x})\}.
\end{equation} 
We can derive the probability that $G_M(\theta_{1}^{m},q_{Y|x})$ does not exceed a constant $c$ as follows: 
\begin{equation}\begin{aligned}
\Pr(G_M&(\theta_1^{m},q_{Y|x})\le c) \\
&= 1-\Pr(\cap_{i=1}^m \{\mathbb E_X G(\theta_i,q_{Y|x})\ge c\})\\
&\ge 1-\Pr(\cap_{i=1}^{m-1} \{\mathbb E_X G(\theta_i,q_{Y|x})\ge c\})\\
&=\Pr(G_M(\theta_1^{m-1},q_{Y|x})\le c).
\end{aligned}\end{equation} 
The derivation indicates that, under reasonable assumptions, an increased number of parameters is advantageous for reducing $G_M(\theta_{1}^m,q_{Y|x})$. 
Consequently, we can formulate the following corollary. 
\begin{corollary}[Effect of parameter quantity]
\label{cor:parameter_quantity}
Under reasonable assumptions, an increase in the number of parameters tends to yield a smaller normalized fitting error by minimizing the lower bound of $\mathbb E_X G(\theta,q_{Y|x})$.
\end{corollary}

\subsection{Non-convex optimization}
In this subsection, we use the derived upper bound on normalized fitting error to analyze how gradient descent algorithm can reduce fitting error and make it tend towards a value that only depends on the model structure.
According to Theorem~\ref{prop:fitting_error}, the bound on fitting error comprises two components: $\mathbb E_{X}G(\theta,q_{Y|x})$ and  $\mathbb E_{X} F(\theta,q_{Y|x})$.
Usually, the $F(\theta,q_{Y|x})$ exhibits non-convexity with respect to $\theta$.
The term $\|\nabla_\theta D_{KL}(q_{Y|x}\|p_{Y|x})\|_2^2$ in $F(\theta,q_{Y|x})$ represents the gradient of the KL divergence between $q_{Y|x}$ and $p_{Y|x}$ with respect to $\theta$. 
Since the function $\ell(f_\theta(x),y)$ is smooth and differentiable with respect to $\theta$, gradient-based algorithms iterate the parameters toward the stationary points where $\mathbb E_{X,Y}\nabla_\theta \ell(f_\theta(x),y)=0$. 
Consequently, if cross-entropy or KL divergence is used as the loss function for model training, gradient-based algorithms can effectively reduce $\|\nabla_\theta D_{KL}(q_{Y|x}\|p_{Y|x})\|_2^2$ for each input feature $x$. 
When the gradient-based algorithm converges to the stationary point, i.e., $F(\theta,q_{Y|x})=0,\forall x\in \mathcal{X}$, we obtain: 
\begin{equation}
\begin{aligned}
\mathrm{fit_n}(\ell,f_\theta)= \sqrt{\mathbb E_X[G(\theta,q_{Y|x})]},
\end{aligned}
\end{equation}
where $G(\theta,q_{Y|x})$ is controlled by the number of parameters as discussed in subsection~\ref{subsec:op}.
Consequently, employing cross entropy or KL divergence as the loss function enables gradient-based algorithms to minimize the upper bound on the normalized fitting error.

\subsection{Relationship between the Bound on Fitting Error and Flat Minima theory}
In this subsection, from our derived upper bound on fitting error, we can infer theoretically that an increase in the maximum eigenvalue of the structural matrix corresponds to enhanced performance, aligning with the principles posited by the flat minima theory. 

In the flat minima theory, it is posited that flat minima generalize better than sharp (non-flat) minima, and the largest eigenvalue of the Hessian of the loss with respect to the model parameters on the training set serves as a metric for flatness. 
Similar to flat minima but with a twist, we have demonstrated that decreasing the largest eigenvalue of structural matrix helps in reducing fitting errors. 
We now proceed to establish the equivalence between the two concepts, thereby indirectly substantiating the efficacy and validity of our proposed bound.
Prior to presenting the theorem, we provide the following lemmas, which are instrumental in proving the theorem.
\begin{lemma}\label{lemma:eigenvalue_diag}
    Let $A=(a_{ij})_{m\times m}$ be a symmetric matrix. Then, we have 
\begin{equation}
    a_{ii}\le \lambda_{\max}(A)\le \mathrm{trace}(A) \le m \lambda_{\max}(A),\forall i\le m.
\end{equation}
\end{lemma}
\begin{proof}
    
Given $A=(a_{ij})_{m\times m}$ be a symmetric matrix, then we have $\lambda_{\min}(A)\ge 0$. 
Since the sum of eigenvalues is equal to $\mathrm{trace}(A)$, thus we have 
\begin{equation}
    \lambda_{\max}(A)\le \mathrm{trace}(A) \le m \lambda_{\max}(A).
\end{equation}
By the Courant-Fisher min-max theorem, we have
\begin{equation}
    \lambda_{\max}(A)=\max_{x\in \mathbb R^m,\|x\|=1}x^\top Ax.
\end{equation}
Let $e_i$ be a standard basis ("one-hot") vector, then we have 
\begin{equation}
    \begin{aligned}
        a_{ii}=e_i^\top Ae_i\le \lambda_{\max}(A).
    \end{aligned}
\end{equation}
\end{proof}

\begin{lemma}\label{lemma:eigenvalue}
        Let $A,B$ be non-negative definite matrices. Then we have 
\begin{equation}
    \begin{aligned}
        \lambda_{\min}(A+B)&\ge \lambda_{\min}(A)+\lambda_{\min}(B),\\
        \lambda_{\max}(A+B)&\le \lambda_{\max}(A)+\lambda_{\max}(B).
    \end{aligned}
\end{equation}
\end{lemma}
\begin{proof}
    Because \begin{equation}
        \begin{aligned}
            &x^\top Ax\ge \lambda_{\min}(A)x^\top x,\\
            &x^\top(A-\lambda_{\min}(A)I)x\ge 0,
        \end{aligned}
    \end{equation}
thus $A-\lambda_{\min}(A)I$ is a non-negative definite matrix.
Therefore, we have 
\begin{equation}
    \begin{aligned}
    &A-\lambda_{\min}(A)I\succeq 0,\\
    &B-\lambda_{\min}(B)I\succeq 0,\\
    &A-\lambda_{\min}(A)I+B-\lambda_{\min}(B)I\succeq 0.
    \end{aligned}
\end{equation}
Let $(A+B)x'=\lambda_{\min}(A+B)x'$, then we have
\begin{equation}
    \begin{aligned}
        (A-\lambda_{\min}(A)I+B&-\lambda_{\min}(B)I)x'=(\lambda_{\min}(A+B)\\
        &-\lambda_{\min}(A)-\lambda_{\min}(B))x'.
    \end{aligned}
\end{equation}
Therefore $\lambda_{\min}(A+B)-\lambda_{\min}(A)-\lambda_{\min}(B)$ is the eigenvalue of $A-\lambda_{\min}(A)I+B-\lambda_{\min}(B)I$.
Because $A-\lambda_{\min}(A)I+B-\lambda_{\min}(B)I$ is a non-negative definite matrix, thus we have $\lambda_{\min}(A+B)-\lambda_{\min}(A)-\lambda_{\min}(B)\ge 0$.

Because $x^\top A\le \lambda_{\max}(A)x^\top x$, then $x^\top(A-\lambda_{\max}(A)I)x\le 0$, thus $\lambda_{\max}(A)I-A$ is a non-negative definite matrix.
Similarly, we can obtain $\lambda_{\max}(A+B)-\lambda_{\max}(A)-\lambda_{\max}(B)\le 0$.
\end{proof}

\begin{proposition}
\label{prop:NTK_ERF}
If cross-entropy or KL divergence is employed as the loss function, and the following equations hold 
\begin{equation}
    \begin{aligned}
&\sum_{i=1}^{|\mathcal{Y}|} p_{Y|x}(y_i)\nabla_{\theta_i} f_{\theta}(x)_i = 0,\\
&\mathbb E_{X\sim q_X}\sum_{i=1}^{|\mathcal{Y}|} p_{Y|x}(y_i)\frac{\nabla_\theta^2 p_{Y|x}(y_i)}{p_{Y|x}(y_i)}= 0,
    \end{aligned}
\end{equation}
after training, then we have:\begin{equation}
    \begin{aligned}
        \lambda_{\max}(H)\le \max_{x}\lambda_{\max}(S(\theta,x)), 
    \end{aligned}
\end{equation}
where $H$ is the Hessian of the loss function with respect to the parameter $\theta$.
\end{proposition}
\begin{proof}
According to the definition, we have $p_{Y|x}(y_i)=\frac{e^{[f_\theta(x)]_i}}{Z_\theta(x)}$, $Z_\theta(x)$ is the partition function, defined as $Z_\theta(x) = \sum_{j=1}^{|\mathcal{Y}|} e^{f_{\theta}(x)_j}$.
Consequently, the gradient of the log-probability with respect to the model parameter $\theta$ decomposes as follows:
\begin{equation}
\nabla_\theta \log p_{Y|x}(y)=\nabla_\theta [f_\theta(x)]_j-\nabla_\theta \log Z_\theta(x),
\end{equation} 
The second gradient term, $\nabla_\theta \log Z_\theta(x)$ can be rewritten as an expected value, as follows:
\begin{equation}\begin{aligned}
    \nabla_\theta \log Z_\theta(x)&=Z_\theta(x)^{-1}\nabla_\theta \sum_{j=1}^{|\mathcal{Y}|} e^{f_\theta(y)_j} \\
    &=Z_\theta(x)^{-1} \sum_{j=1}^{|\mathcal{Y}|} e^{f_\theta(y)_j}\nabla_\theta f_\theta(y)_j\\
    &=\sum_{j=1}^{|\mathcal{Y}|} q_{Y|x}(y_j) \nabla_\theta [f_\theta(x)]_j. 
\end{aligned}
\end{equation}
According to the above equality, we have 
\begin{equation}
    \begin{aligned}
        \sum_{i=1}^{|\mathcal{Y}|} q_{Y|x}(y_i)&\left\{\nabla_\theta \log p_{Y|x}(y)\right\}\left\{\nabla_\theta \log p_{Y|x}(y)\right\}^\top\\
        &=\sum_{i=1}^{|\mathcal{Y}|} q_{Y|x}(y_i)\{\nabla_{\theta} [f_\theta(x)]_i-\nabla_{\theta}\log Z_\theta(x)\}\\
        &\{\nabla_{\theta} [f_\theta(x)]_i-\nabla_{\theta}\log Z_\theta(x)\}^\top\\
        &=\sum_{i=1}^{|\mathcal{Y}|} q_{Y|x}(y_i)\{\nabla_{\theta} [f_\theta(x)]_i\nabla_{\theta} [f_\theta(x)]_i^\top\}\\
        &-2\nabla_{\theta}\log Z_\theta(x)\mathbb E_{Y\sim q_{Y|x}} \nabla_{\theta} [f_\theta(x)]_i\\
        &+\nabla_{\theta}\log Z_\theta(x)\nabla_{\theta}\log Z_\theta(x)^\top\\
        &=\sum_{i=1}^{|\mathcal{Y}|} q_{Y|x}(y_i)\{\nabla_{\theta} [f_\theta(x)]_i\nabla_{\theta} [f_\theta(x)]_i^\top\}+F_x,
    \end{aligned}
\end{equation}
where 
\begin{equation}
    \begin{aligned}
        F_x &= -2\nabla_{\theta}\log Z_\theta(x)\sum_{i=1}^{|\mathcal{Y}|} q_{Y|x}(y_i) \nabla_{\theta} [f_\theta(x)]_i\\
        &+\nabla_{\theta}\log Z_\theta(x)\nabla_{\theta}\log Z_\theta(x)^\top.
    \end{aligned}
\end{equation}

The Hessian for $D_{KL}(q_{Y|x},p_{Y|x})$ is equal to computing the second derivative of the negative log-likelihood $-\log p_{Y|x}(y)$, which can be expressed as:
\begin{equation}\label{eq:hessian_4}
    \nabla^2_\theta D_{KL}(q_{Y|x},p_{Y|x})=\sum_{i=1}^{|\mathcal{Y}|} q_{Y|x}(y_i)[-\nabla^2_\theta \log p_{Y|x}(y_i)].
\end{equation}
Next, we expand the second derivative as follows: 
\begin{equation}\label{eq:hessian_5}
    \begin{aligned}
-\nabla_\theta^2 &\log p_{Y|x}(y)  =-\nabla_\theta\left[\frac{\nabla_\theta p_{Y|x}(y)}{p_{Y|x}(y)}\right] \\
& =-\frac{p_{Y|x}(y) \nabla_\theta^2 p_{Y|x}(y)-\left\{\nabla_\theta p_{Y|x}(y)\right\}\left\{\nabla_\theta p_{Y|x}(y)\right\}^\top}{p_{Y|x}(y)^2} \\
& =-\frac{\nabla_\theta^2 p_{Y|x}(y)}{p_{Y|x}(y)}+\left\{\frac{\nabla_\theta p_{Y|x}(y)}{p_{Y|x}(y)}\right\}\left\{\frac{\nabla_\theta p_{Y|x}(y)}{p_{Y|x}(y)}\right\}^\top \\
& =-\frac{\nabla_\theta^2 p_{Y|x}(y)}{p_{Y|x}(y)}+\left\{\nabla_\theta \log p_{Y|x}(y)\right\}\left\{\nabla_\theta \log p_{Y|x}(y)\right\}^\top.
\end{aligned}
\end{equation}

Then, we apply this expanded second derivative to~\eqref{eq:hessian_4} which can be restated as follows: 
\begin{equation}\label{eq:Hessian_fisher_6}
    \begin{aligned}
\nabla^2_\theta &D_{KL}(q_{Y|x},p_{Y|x}) =\sum_{i=1}^{|\mathcal{Y}|} q_{Y|x}(y_i)[-\frac{\nabla_\theta^2 p_{Y|x}(y_i)}{p_{Y|x}(y_i)}\\
&+\left\{\nabla_\theta \log p_{Y|x}(y_i)\right\}\left\{\nabla_\theta \log p_{Y|x}(y_i)\right\}^\top] \\
& =C_x+B_x+F_x,
\end{aligned}
\end{equation}
where 
\begin{equation}
    \begin{aligned}
        B_x &=\sum_{i=1}^{|\mathcal{Y}|} q_{Y|x}(y_i)\nabla_{\theta} [f_\theta(x)]_i\nabla_{\theta} [f_\theta(x)]_i^\top,\\
        C_x&=\sum_{i=1}^{|\mathcal{Y}|} q_{Y|x}(y_i)\left[-\frac{\nabla_\theta^2 p_{Y|x}(y_i)}{p_{Y|x}(y_i)}\right].
    \end{aligned}
\end{equation}

Since $\nabla_{\theta} [f_\theta(x)]_i\nabla_{\theta} [f_\theta(x)]_i^\top$ is a matrix of rank 1, its eigenvalues are
\begin{equation}
    \begin{aligned}
        \lambda_{\max}(\nabla_{\theta} [f_\theta(x)]_i\nabla_{\theta} [f_\theta(x)]_i^\top)&=\nabla_{\theta} [f_\theta(x)]_i^\top \nabla_{\theta} [f_\theta(x)]_i,\\
        \lambda_{\min}(\nabla_{\theta} [f_\theta(x)]_i\nabla_{\theta} [f_\theta(x)]_i^\top)&=0.
    \end{aligned}
\end{equation}

According to lemma~\ref{lemma:eigenvalue}, we can bound the extreme eigenvalues of $B_x$ as follows: 
\begin{equation}\label{eq:temp_eigen}
    \begin{aligned}
        \lambda_{\max}(B_x)&\le \sum_{i=1}^{|\mathcal{Y}|} q_{Y|x}(y_i) \lambda_{\max}(\nabla_{\theta} [f_\theta(x)]_i\nabla_{\theta} [f_\theta(x)]_i^\top)\\
        &\le\max_{y}\lambda_{\max}(\nabla_{\theta} [f_\theta(x)]_i\nabla_{\theta} [f_\theta(x)]_i^\top)\\
        &=\max_y \nabla_{\theta} [f_\theta(x)]_i^\top \nabla_{\theta} [f_\theta(x)]_i,\\
        \lambda_{\min}(B_x)&\ge \sum_{i=1}^{|\mathcal{Y}|} q_{Y|x}(y_i) \lambda_{\min}(\nabla_{\theta} [f_\theta(x)]_i\nabla_{\theta} [f_\theta(x)]_i^\top)\\
        &\ge\min_{y}\lambda_{\min}(\nabla_{\theta} [f_\theta(x)]_i\nabla_{\theta} [f_\theta(x)]_i^\top)\\
        &=0.
    \end{aligned}
\end{equation}

Given that $S(\theta,x)= \nabla_\theta f_\theta(x)^\top \nabla_\theta f_\theta(x)$, and since $S(\theta,x)$ is a symmetric matrix, according to the inequality~\eqref{eq:temp_eigen} and lemma~\ref{lemma:eigenvalue_diag}, we have: 
\begin{equation}
    \begin{aligned}
        \lambda_{\max}(B_x)\le \max_y \nabla_{\theta} [f_\theta(x)]_i^\top \nabla_{\theta} [f_\theta(x)]_i &\le \lambda_{\max}(S(\theta,x)).
    \end{aligned}
\end{equation}
The Hessian for $\mathbb E_{X\sim q_X}D_{KL}(q_{Y|x},p_{Y|x})$ can be denoted as:
\begin{equation}
    \begin{aligned}
        &H=\mathbb E_{X\sim q_X} B_x + \mathbb E_{X\sim q_X} C_x +\mathbb E_{X\sim q_X} F_x. 
    \end{aligned}
\end{equation}
Since 
\begin{equation}
    \begin{aligned}
        &\sum_{i=1}^{|\mathcal{Y}|} q_{Y|x}(y_i) \nabla_\theta [f_\theta(x)]_i= 0,\\
        &\mathbb E_{X\sim q_X}\sum_{i=1}^{|\mathcal{Y}|} q_{Y|x}(y_i)\frac{\nabla_\theta^2 p_{Y|x}(y_i)}{p_{Y|x}(y_i)}= 0,
    \end{aligned}
\end{equation}
and since $H= \mathbb E_{X\sim q_X} B_x$, we deduce
\begin{equation}
    \begin{aligned}
        \lambda_{\max}(H)\le \max_{x}\lambda_{\max}(B_x)\le \max_{x}\lambda_{\max}(S(\theta,x)).
    \end{aligned}
\end{equation}
\end{proof}

As shown in proposition~\ref{prop:NTK_ERF}, minimizing $\lambda_{\max}(S(\theta,x)),\forall x\in \mathcal{X}$, is equivalent to minimizing $\lambda_{\max}(H)$ and our proposed bound on fitting error provides an alternative theoretical explanation for the flat minima theory.

\section{Empirical Validations}
\label{sec:experiments}

\begin{figure*}[!htb]
\centering
\includegraphics[width=0.6\linewidth]{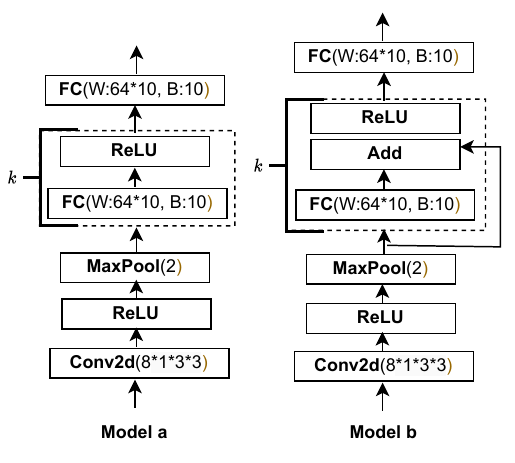}
\caption{Model architectures and configuration parameters.}
\label{fig:model_increase}
\end{figure*}

\begin{figure*}[!htb]
\centering
\includegraphics[width=0.9\linewidth]{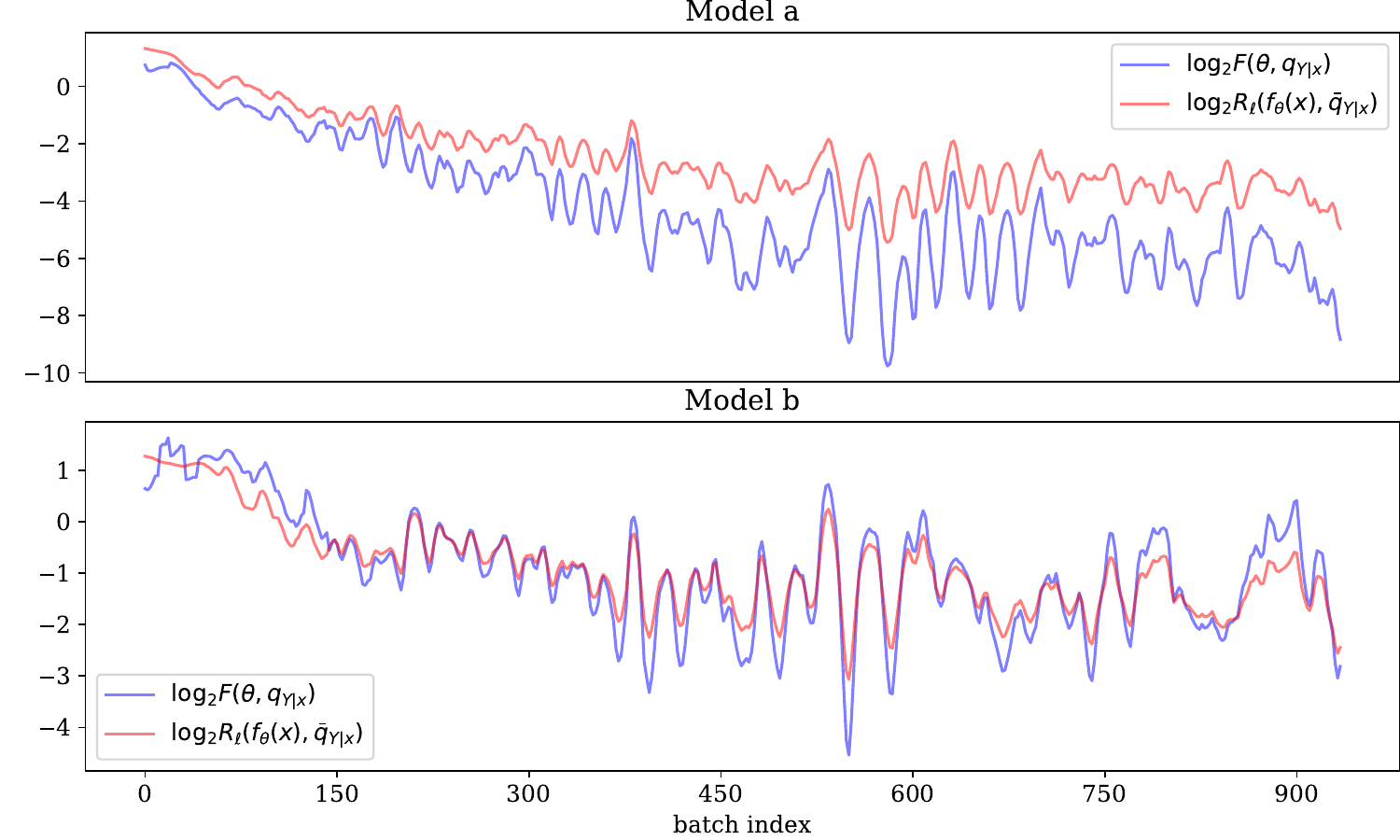}
\caption{Model architectures and configuration parameters.}
\label{fig:convergence_indicator}
\end{figure*}

\begin{figure*}[!htb]
\centering
\includegraphics[width=0.9\linewidth]{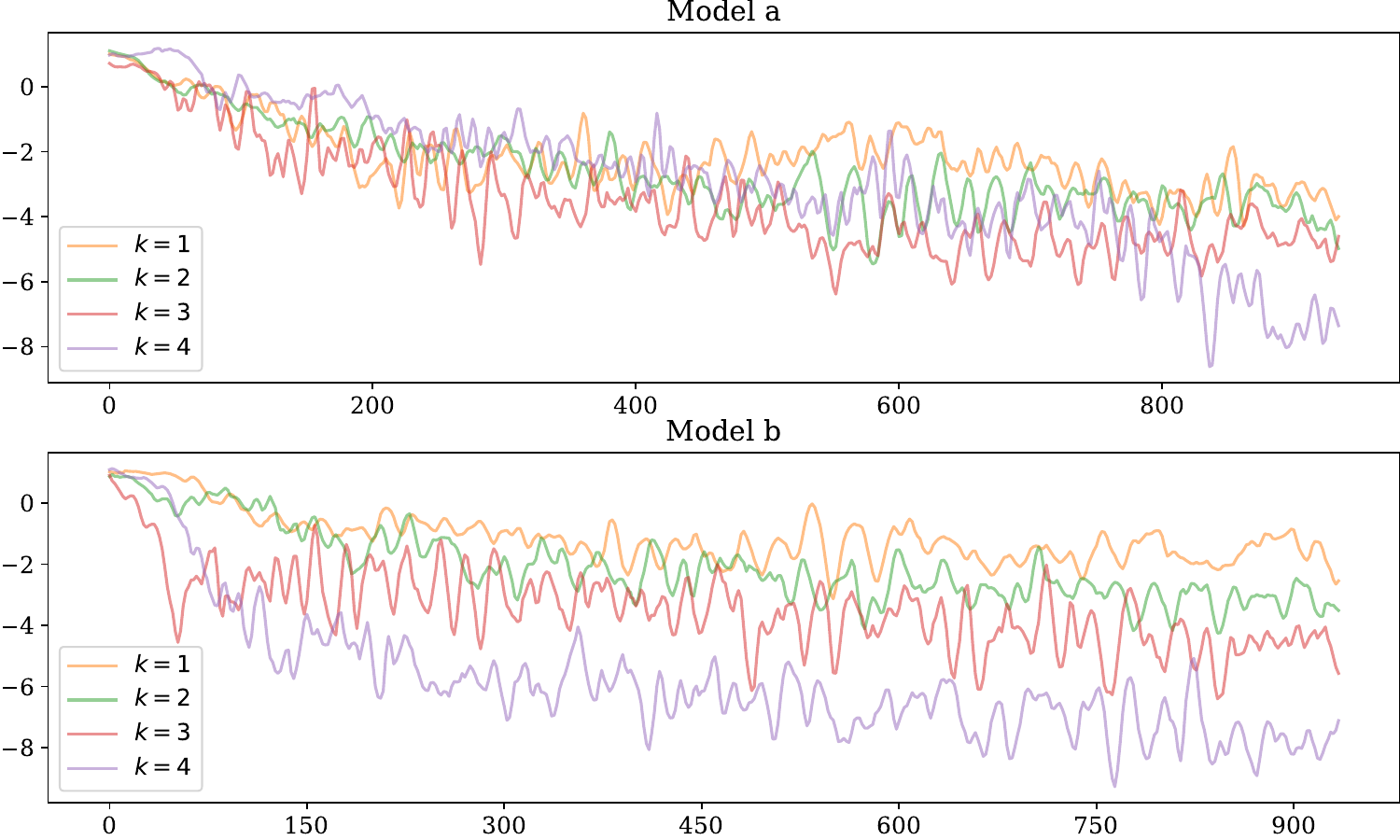}
\caption{Model architectures and configuration parameters.}
\label{fig:layer_convergence}
\end{figure*}

In the preceding section, we theoretically examine the proposed expected risk bound. This section aims to empirically validate its effectiveness.
Given the same classification task, the inherent classification uncertainty remains constant, thus we concentrate on the fitting error aspect. 
Within the upper bound of the fitting error, $G(\theta,q_{Y|x})$ exhibits stochastic characteristics during the gradient descent iterations, while $F(\theta,q_{Y|x})$ is consistent with the currently popular flat minima theory under reasonable assumptions, as demonstrated in proposition~\ref{prop:NTK_ERF}. 
More importantly, the item  $F(\theta,q_{Y|x})$ is both controllable and optimizable during the model training process. 
Consequently, the key to validating the effectiveness of our proposed upper bound lies in confirming the correlation between $F(\theta,q_{Y|x})$ and the expected risk $R_{\ell}(f_\theta,\bar q)$. 
In this empirical validation, we therefore assess the effectiveness of the proposed bound by observing the correlation between the changes in the loss on the test set and the variations in $F(\theta,q_{Y|x})$ during the training process. 
Furthermore, we will examine the variation in the discrepancy between $R_{\ell}(f_\theta,\bar q)$ and $F(\theta,q_{Y|x})$ by increasing the number of network layers and parameters. This investigation aims to substantiate the conclusion drawn in Subsection~\ref{subsec:op} , which posits that augmenting the parameter count facilitates the reduction of $G(\theta,q_{Y|x})$, thereby narrowing the gap between $R_{\ell}(f_\theta,\bar q)$ and $F(\theta,q_{Y|x})$.

In our experiments, we utilized the MNIST database~\cite{Yang2019CARSCE} as a case study to visualize how the structural error, gradient norm and fitting error evolve as we increase the network depth across different model architectures. 
This database has 60k training images and 10k testing images in 10 categories, containing 10 handwritten digits. 
In terms of modeling, we selected two types of typical network architectures: standard network and those incorporating residual blocks. Both are widely utilized frameworks in the field of deep learning.
The model architectures and configuration parameters in the experiment are depicted in Figures~\ref{fig:model_increase}. 

The experiments were executed using the following computing environment: Python 3.7, Pytorch 2.2.2, and a GeForce RTX2080Ti GPU. 
The training parameters are detailed in Table~\ref{tab:config}. 
\begin{table}[!htb]
\caption{The configuration of Training.}
\label{tab:config}
\centering
\begin{tabular}{l l l l}
\hline
loss function& cross-entropy loss & \#epochs & 1 \\
\hline
batch size & 64 & learning rate& 0.01\\
\hline
momentum & 0.9 & data size& 28*28 \\

\hline

\end{tabular}

\end{table}

We assign a value of $k=1$ to both Model a and Model b and proceed to train these models. The variations in $F(\theta,q_{Y|x})$ and $R_{\ell}(f_\theta(x),\bar q_{Y|x})$ for any arbitrary $x$ throughout the training process are depicted in Figures~\ref{fig:convergence_indicator}. 

As illustrated in Figures~\ref{fig:convergence_indicator}, for both Model a and Model b, the trend of change in $R_{\ell}(f_\theta(x),\bar q_{Y|x})$ and $F(\theta,q_{Y|x})$ remains entirely consistent as the models progress toward convergence. 
In other words, the expected risk is completely governed by its upper bound $F(\theta,q_{Y|x})$. This experimental finding validates that the upper bound of the expected risk we proposed is sufficiently tight and effective.
Furthermore, upon comparing the discrepancies between the curves of Model a and Model b, it is unsurprising that Model b converges more rapidly due to the presence of residual blocks. However, it is noteworthy that Model a achieves a smaller expected risk than Model b. 
We believe that residual blocks, via skip connections, prevent the gradient from becoming excessively large or small, which results in a larger $F(\theta,q_{Y|x})$ and expected risk for Model b compared to Model a.

Since $G(\theta,q_{Y|x})$ and $F(\theta,q_{Y|x})$ jointly constitute the upper bound of $R_{\ell}(f_\theta(x),\bar q_{Y|x})$, the aforementioned experiment has demonstrated that this upper bound is sufficiently tight when the model converges. Therefore, it is reasonable to approximate $G(\theta,q_{Y|x})$ with $R_{\ell}(f_\theta(x),\bar q_{Y|x})-F(\theta,q_{Y|x})$ at the point of model convergence. To investigate the impact of the parameter count on $G(\theta,q_{Y|x})$, we increase the number of network layers $k$ from 1 to 4 and observe the curve of $\log (R_{\ell}(f_\theta(x),\bar q_{Y|x})-F(\theta,q_{Y|x}))$ as depicted in the following Figure~\ref{fig:layer_convergence}.

As illustrated in Figures~\ref{fig:layer_convergence}, upon the convergence of the losses on the training dataset, the greater the number of layers in the model, the smaller the corresponding $\log (R_{\ell}(f_\theta(x),\bar q_{Y|x})-F(\theta,q_{Y|x}))$ becomes. This implies that $G(\theta,q_{Y|x})$ diminishes with the increase in the parameter count, thereby corroborating the conclusion drawn in Subsection~\ref{subsec:op}. 
From this perspective, the role of overparameterization is not to alter the model's generalization capability but to tighten the upper bound of the fitting error by reducing $G(\theta,q_{Y|x})$.


\section{Conclusion}
\label{sec:conclusion}
This study introduces a novel upper bound for expected risk in supervised classification, incorporating the concepts of fitting error and model risk. The derived upper bound connects the gradient norm and the model's parameter count to the fitting error, providing valuable insights into the overparameterization, non-convex optimization inherent in DNNs and flat minima theory.
Moreover, we establish a relationship between the generalization error, classification uncertainty, and the smoothness of the dataset distribution, proposing that classification uncertainty can serve as a metric for the reliability of the training dataset and can assist in the determination of hyperparameters for regularization. 
Additionally, our experimental results confirm a significant positive correlation between the proposed upper bounds and the expected risk encountered in practical applications. 
Thus, based on the new upper bound, it becomes possible to achieve better model performance and inspire the development of new algorithms by controlling the key factors in the upper bound. 

\backmatter

\bibliography{sn-bibliography}


\begin{thebibliography}{59}
\ifx \bisbn   \undefined \def \bisbn  #1{ISBN #1}\fi
\ifx \binits  \undefined \def \binits#1{#1}\fi
\ifx \bauthor  \undefined \def \bauthor#1{#1}\fi
\ifx \batitle  \undefined \def \batitle#1{#1}\fi
\ifx \bjtitle  \undefined \def \bjtitle#1{#1}\fi
\ifx \bvolume  \undefined \def \bvolume#1{\textbf{#1}}\fi
\ifx \byear  \undefined \def \byear#1{#1}\fi
\ifx \bissue  \undefined \def \bissue#1{#1}\fi
\ifx \bfpage  \undefined \def \bfpage#1{#1}\fi
\ifx \blpage  \undefined \def \blpage #1{#1}\fi
\ifx \burl  \undefined \def \burl#1{\textsf{#1}}\fi
\ifx \doiurl  \undefined \def \doiurl#1{\url{https://doi.org/#1}}\fi
\ifx \betal  \undefined \def \betal{\textit{et al.}}\fi
\ifx \binstitute  \undefined \def \binstitute#1{#1}\fi
\ifx \binstitutionaled  \undefined \def \binstitutionaled#1{#1}\fi
\ifx \bctitle  \undefined \def \bctitle#1{#1}\fi
\ifx \beditor  \undefined \def \beditor#1{#1}\fi
\ifx \bpublisher  \undefined \def \bpublisher#1{#1}\fi
\ifx \bbtitle  \undefined \def \bbtitle#1{#1}\fi
\ifx \bedition  \undefined \def \bedition#1{#1}\fi
\ifx \bseriesno  \undefined \def \bseriesno#1{#1}\fi
\ifx \blocation  \undefined \def \blocation#1{#1}\fi
\ifx \bsertitle  \undefined \def \bsertitle#1{#1}\fi
\ifx \bsnm \undefined \def \bsnm#1{#1}\fi
\ifx \bsuffix \undefined \def \bsuffix#1{#1}\fi
\ifx \bparticle \undefined \def \bparticle#1{#1}\fi
\ifx \barticle \undefined \def \barticle#1{#1}\fi
\bibcommenthead
\ifx \bconfdate \undefined \def \bconfdate #1{#1}\fi
\ifx \botherref \undefined \def \botherref #1{#1}\fi
\ifx \url \undefined \def \url#1{\textsf{#1}}\fi
\ifx \bchapter \undefined \def \bchapter#1{#1}\fi
\ifx \bbook \undefined \def \bbook#1{#1}\fi
\ifx \bcomment \undefined \def \bcomment#1{#1}\fi
\ifx \oauthor \undefined \def \oauthor#1{#1}\fi
\ifx \citeauthoryear \undefined \def \citeauthoryear#1{#1}\fi
\ifx \endbibitem  \undefined \def \endbibitem {}\fi
\ifx \bconflocation  \undefined \def \bconflocation#1{#1}\fi
\ifx \arxivurl  \undefined \def \arxivurl#1{\textsf{#1}}\fi
\csname PreBibitemsHook\endcsname

\bibitem[\protect\citeauthoryear{Valiant}{1984}]{valiant1984theory}
\begin{barticle}
\bauthor{\bsnm{Valiant}, \binits{L.G.}}:
\batitle{A theory of the learnable}.
\bjtitle{Communications of the ACM}
\bvolume{27}(\bissue{11}),
\bfpage{1134}--\blpage{1142}
(\byear{1984})
\end{barticle}
\endbibitem

\bibitem[\protect\citeauthoryear{Bousquet and Elisseeff}{2002}]{Bousquet2002StabilityAG}
\begin{barticle}
\bauthor{\bsnm{Bousquet}, \binits{O.}},
\bauthor{\bsnm{Elisseeff}, \binits{A.}}:
\batitle{Stability and generalization}.
\bjtitle{Journal of Machine Learning Research}
\bvolume{2},
\bfpage{499}--\blpage{526}
(\byear{2002})
\end{barticle}
\endbibitem

\bibitem[\protect\citeauthoryear{Mukherjee et~al.}{2006}]{mukherjee2006learning}
\begin{barticle}
\bauthor{\bsnm{Mukherjee}, \binits{S.}},
\bauthor{\bsnm{Niyogi}, \binits{P.}},
\bauthor{\bsnm{Poggio}, \binits{T.}},
\bauthor{\bsnm{Rifkin}, \binits{R.}}:
\batitle{Learning theory: stability is sufficient for generalization and necessary and sufficient for consistency of empirical risk minimization}.
\bjtitle{Advances in Computational Mathematics}
\bvolume{25},
\bfpage{161}--\blpage{193}
(\byear{2006})
\end{barticle}
\endbibitem

\bibitem[\protect\citeauthoryear{Kawaguchi et~al.}{2017}]{Kawaguchi2017GeneralizationID}
\begin{botherref}
\oauthor{\bsnm{Kawaguchi}, \binits{K.}},
\oauthor{\bsnm{Kaelbling}, \binits{L.P.}},
\oauthor{\bsnm{Bengio}, \binits{Y.}}:
Generalization in deep learning.
ArXiv
\textbf{abs/1710.05468}
(2017)
\end{botherref}
\endbibitem

\bibitem[\protect\citeauthoryear{Leshno et~al.}{1993}]{Leshno1993OriginalCM}
\begin{barticle}
\bauthor{\bsnm{Leshno}, \binits{M.}},
\bauthor{\bsnm{Lin}, \binits{V.Y.}},
\bauthor{\bsnm{Pinkus}, \binits{A.}},
\bauthor{\bsnm{Schocken}, \binits{S.}}:
\batitle{Original contribution: Multilayer feedforward networks with a nonpolynomial activation function can approximate any function}.
\bjtitle{Neural Networks}
\bvolume{6},
\bfpage{861}--\blpage{867}
(\byear{1993})
\end{barticle}
\endbibitem

\bibitem[\protect\citeauthoryear{Barron}{1993}]{Barron1993UniversalAB}
\begin{barticle}
\bauthor{\bsnm{Barron}, \binits{A.R.}}:
\batitle{Universal approximation bounds for superpositions of a sigmoidal function}.
\bjtitle{IEEE Trans. Inf. Theory}
\bvolume{39},
\bfpage{930}--\blpage{945}
(\byear{1993})
\end{barticle}
\endbibitem

\bibitem[\protect\citeauthoryear{Choromańska et~al.}{2014}]{Choromaska2014TheLS}
\begin{bchapter}
\bauthor{\bsnm{Choromańska}, \binits{A.}},
\bauthor{\bsnm{Henaff}, \binits{M.}},
\bauthor{\bsnm{Mathieu}, \binits{M.}},
\bauthor{\bsnm{Arous}, \binits{G.B.}},
\bauthor{\bsnm{LeCun}, \binits{Y.}}:
\bctitle{The loss surfaces of multilayer networks}.
In: \bbtitle{International Conference on Artificial Intelligence and Statistics}
(\byear{2014}).
\burl{https://api.semanticscholar.org/CorpusID:2266226}
\end{bchapter}
\endbibitem

\bibitem[\protect\citeauthoryear{Kawaguchi}{2016}]{Kawaguchi2016DeepLW}
\begin{botherref}
\oauthor{\bsnm{Kawaguchi}, \binits{K.}}:
Deep learning without poor local minima.
ArXiv
\textbf{abs/1605.07110}
(2016)
\end{botherref}
\endbibitem

\bibitem[\protect\citeauthoryear{Wu}{2021}]{Wu2021StatisticalLT}
\begin{barticle}
\bauthor{\bsnm{Wu}, \binits{Y.}}:
\batitle{Statistical learning theory}.
\bjtitle{Technometrics}
\bvolume{41},
\bfpage{377}--\blpage{378}
(\byear{2021})
\end{barticle}
\endbibitem

\bibitem[\protect\citeauthoryear{Bartlett and Mendelson}{2003}]{Bartlett2003RademacherAG}
\begin{barticle}
\bauthor{\bsnm{Bartlett}, \binits{P.L.}},
\bauthor{\bsnm{Mendelson}, \binits{S.}}:
\batitle{Rademacher and gaussian complexities: Risk bounds and structural results}.
\bjtitle{J. Mach. Learn. Res.}
\bvolume{3},
\bfpage{463}--\blpage{482}
(\byear{2003})
\end{barticle}
\endbibitem

\bibitem[\protect\citeauthoryear{Zhang et~al.}{2016}]{Zhang2016UnderstandingDL}
\begin{botherref}
\oauthor{\bsnm{Zhang}, \binits{C.}},
\oauthor{\bsnm{Bengio}, \binits{S.}},
\oauthor{\bsnm{Hardt}, \binits{M.}},
\oauthor{\bsnm{Recht}, \binits{B.}},
\oauthor{\bsnm{Vinyals}, \binits{O.}}:
Understanding deep learning requires rethinking generalization.
ArXiv
\textbf{abs/1611.03530}
(2016)
\end{botherref}
\endbibitem

\bibitem[\protect\citeauthoryear{Zhang et~al.}{2021}]{Zhang2021UnderstandingDL}
\begin{barticle}
\bauthor{\bsnm{Zhang}, \binits{C.}},
\bauthor{\bsnm{Bengio}, \binits{S.}},
\bauthor{\bsnm{Hardt}, \binits{M.}},
\bauthor{\bsnm{Recht}, \binits{B.}},
\bauthor{\bsnm{Vinyals}, \binits{O.}}:
\batitle{Understanding deep learning (still) requires rethinking generalization}.
\bjtitle{Communications of the ACM}
\bvolume{64},
\bfpage{107}--\blpage{115}
(\byear{2021})
\end{barticle}
\endbibitem

\bibitem[\protect\citeauthoryear{He et~al.}{2015}]{He2015DeepRL}
\begin{botherref}
\oauthor{\bsnm{He}, \binits{K.}},
\oauthor{\bsnm{Zhang}, \binits{X.}},
\oauthor{\bsnm{Ren}, \binits{S.}},
\oauthor{\bsnm{Sun}, \binits{J.}}:
Deep residual learning for image recognition.
2016 IEEE Conference on Computer Vision and Pattern Recognition (CVPR),
770--778
(2015)
\end{botherref}
\endbibitem

\bibitem[\protect\citeauthoryear{Belkin et~al.}{2018}]{Belkin2018ReconcilingMM}
\begin{barticle}
\bauthor{\bsnm{Belkin}, \binits{M.}},
\bauthor{\bsnm{Hsu}, \binits{D.J.}},
\bauthor{\bsnm{Ma}, \binits{S.}},
\bauthor{\bsnm{Mandal}, \binits{S.}}:
\batitle{Reconciling modern machine-learning practice and the classical bias–variance trade-off}.
\bjtitle{Proceedings of the National Academy of Sciences}
\bvolume{116},
\bfpage{15849}--\blpage{15854}
(\byear{2018})
\end{barticle}
\endbibitem

\bibitem[\protect\citeauthoryear{Vapnik and Alexey}{1971}]{vapnik1971chervonenkis}
\begin{barticle}
\bauthor{\bsnm{Vapnik}, \binits{V.N.}},
\bauthor{\bsnm{Alexey}, \binits{Y.}}:
\batitle{Chervonenkis. on the uniform convergence of relative frequencies of events to their probabilities}.
\bjtitle{Theory of Probability and its Applications}
\bvolume{16}(\bissue{2}),
\bfpage{264}--\blpage{280}
(\byear{1971})
\end{barticle}
\endbibitem

\bibitem[\protect\citeauthoryear{Keskar et~al.}{2016}]{Keskar2016OnLT}
\begin{botherref}
\oauthor{\bsnm{Keskar}, \binits{N.S.}},
\oauthor{\bsnm{Mudigere}, \binits{D.}},
\oauthor{\bsnm{Nocedal}, \binits{J.}},
\oauthor{\bsnm{Smelyanskiy}, \binits{M.}},
\oauthor{\bsnm{Tang}, \binits{P.T.P.}}:
On large-batch training for deep learning: Generalization gap and sharp minima.
ArXiv
\textbf{abs/1609.04836}
(2016)
\end{botherref}
\endbibitem

\bibitem[\protect\citeauthoryear{Neyshabur et~al.}{2014}]{Neyshabur2014InSO}
\begin{botherref}
\oauthor{\bsnm{Neyshabur}, \binits{B.}},
\oauthor{\bsnm{Tomioka}, \binits{R.}},
\oauthor{\bsnm{Srebro}, \binits{N.}}:
In search of the real inductive bias: On the role of implicit regularization in deep learning.
CoRR
\textbf{abs/1412.6614}
(2014)
\end{botherref}
\endbibitem

\bibitem[\protect\citeauthoryear{Russo and Zou}{2015}]{Russo2015ControllingBI}
\begin{bchapter}
\bauthor{\bsnm{Russo}, \binits{D.}},
\bauthor{\bsnm{Zou}, \binits{J.Y.}}:
\bctitle{Controlling bias in adaptive data analysis using information theory}.
In: \bbtitle{International Conference on Artificial Intelligence and Statistics}
(\byear{2015}).
\burl{https://api.semanticscholar.org/CorpusID:6842925}
\end{bchapter}
\endbibitem

\bibitem[\protect\citeauthoryear{Xu and Raginsky}{2017}]{Xu2017InformationtheoreticAO}
\begin{botherref}
\oauthor{\bsnm{Xu}, \binits{A.}},
\oauthor{\bsnm{Raginsky}, \binits{M.}}:
Information-theoretic analysis of generalization capability of learning algorithms.
ArXiv
\textbf{abs/1705.07809}
(2017)
\end{botherref}
\endbibitem

\bibitem[\protect\citeauthoryear{Raginsky et~al.}{2016}]{Raginsky2016InformationtheoreticAO}
\begin{botherref}
\oauthor{\bsnm{Raginsky}, \binits{M.}},
\oauthor{\bsnm{Rakhlin}, \binits{A.}},
\oauthor{\bsnm{Tsao}, \binits{M.W.}},
\oauthor{\bsnm{Wu}, \binits{Y.}},
\oauthor{\bsnm{Xu}, \binits{A.}}:
Information-theoretic analysis of stability and bias of learning algorithms.
2016 IEEE Information Theory Workshop (ITW),
26--30
(2016)
\end{botherref}
\endbibitem

\bibitem[\protect\citeauthoryear{Pensia et~al.}{2018}]{Pensia2018GeneralizationEB}
\begin{botherref}
\oauthor{\bsnm{Pensia}, \binits{A.}},
\oauthor{\bsnm{Jog}, \binits{V.}},
\oauthor{\bsnm{Loh}, \binits{P.-L.}}:
Generalization error bounds for noisy, iterative algorithms.
2018 IEEE International Symposium on Information Theory (ISIT),
546--550
(2018)
\end{botherref}
\endbibitem

\bibitem[\protect\citeauthoryear{Asadi et~al.}{2018}]{Asadi2018ChainingMI}
\begin{botherref}
\oauthor{\bsnm{Asadi}, \binits{A.-R.}},
\oauthor{\bsnm{Abbe}, \binits{E.}},
\oauthor{\bsnm{Verd{\'u}}, \binits{S.}}:
Chaining mutual information and tightening generalization bounds.
ArXiv
\textbf{abs/1806.03803}
(2018)
\end{botherref}
\endbibitem

\bibitem[\protect\citeauthoryear{Goldfeld et~al.}{2018}]{Goldfeld2018EstimatingIF}
\begin{bchapter}
\bauthor{\bsnm{Goldfeld}, \binits{Z.}},
\bauthor{\bsnm{Berg}, \binits{E.}},
\bauthor{\bsnm{Greenewald}, \binits{K.H.}},
\bauthor{\bsnm{Melnyk}, \binits{I.}},
\bauthor{\bsnm{Nguyen}, \binits{N.H.}},
\bauthor{\bsnm{Kingsbury}, \binits{B.}},
\bauthor{\bsnm{Polyanskiy}, \binits{Y.}}:
\bctitle{Estimating information flow in deep neural networks}.
In: \bbtitle{International Conference on Machine Learning}
(\byear{2018}).
\burl{https://api.semanticscholar.org/CorpusID:170078981}
\end{bchapter}
\endbibitem

\bibitem[\protect\citeauthoryear{Negrea et~al.}{2019}]{Negrea2019InformationTheoreticGB}
\begin{botherref}
\oauthor{\bsnm{Negrea}, \binits{J.}},
\oauthor{\bsnm{Haghifam}, \binits{M.}},
\oauthor{\bsnm{Dziugaite}, \binits{G.K.}},
\oauthor{\bsnm{Khisti}, \binits{A.}},
\oauthor{\bsnm{Roy}, \binits{D.M.}}:
Information-theoretic generalization bounds for sgld via data-dependent estimates.
ArXiv
\textbf{abs/1911.02151}
(2019)
\end{botherref}
\endbibitem

\bibitem[\protect\citeauthoryear{Hellstr{\"o}m and Durisi}{2022}]{Hellstrm2022ANF}
\begin{botherref}
\oauthor{\bsnm{Hellstr{\"o}m}, \binits{F.}},
\oauthor{\bsnm{Durisi}, \binits{G.}}:
A new family of generalization bounds using samplewise evaluated cmi.
ArXiv
\textbf{abs/2210.06422}
(2022)
\end{botherref}
\endbibitem

\bibitem[\protect\citeauthoryear{Zhou et~al.}{2022}]{Zhou2022StochasticCA}
\begin{botherref}
\oauthor{\bsnm{Zhou}, \binits{R.}},
\oauthor{\bsnm{Tian}, \binits{C.}},
\oauthor{\bsnm{Liu}, \binits{T.}}:
Stochastic chaining and strengthened information-theoretic generalization bounds.
2022 IEEE International Symposium on Information Theory (ISIT),
690--695
(2022)
\end{botherref}
\endbibitem

\bibitem[\protect\citeauthoryear{Clerico et~al.}{2022}]{Clerico2022ChainedGB}
\begin{bchapter}
\bauthor{\bsnm{Clerico}, \binits{E.}},
\bauthor{\bsnm{Shidani}, \binits{A.}},
\bauthor{\bsnm{Deligiannidis}, \binits{G.}},
\bauthor{\bsnm{Doucet}, \binits{A.}}:
\bctitle{Chained generalisation bounds}.
In: \bbtitle{Annual Conference Computational Learning Theory}
(\byear{2022}).
\burl{https://api.semanticscholar.org/CorpusID:247218601}
\end{bchapter}
\endbibitem

\bibitem[\protect\citeauthoryear{Bu et~al.}{2019}]{Bu2019TighteningMI}
\begin{botherref}
\oauthor{\bsnm{Bu}, \binits{Y.}},
\oauthor{\bsnm{Zou}, \binits{S.}},
\oauthor{\bsnm{Veeravalli}, \binits{V.V.}}:
Tightening mutual information based bounds on generalization error.
2019 IEEE International Symposium on Information Theory (ISIT),
587--591
(2019)
\end{botherref}
\endbibitem

\bibitem[\protect\citeauthoryear{Haghifam et~al.}{2020}]{Haghifam2020SharpenedGB}
\begin{botherref}
\oauthor{\bsnm{Haghifam}, \binits{M.}},
\oauthor{\bsnm{Negrea}, \binits{J.}},
\oauthor{\bsnm{Khisti}, \binits{A.}},
\oauthor{\bsnm{Roy}, \binits{D.M.}},
\oauthor{\bsnm{Dziugaite}, \binits{G.K.}}:
Sharpened generalization bounds based on conditional mutual information and an application to noisy, iterative algorithms.
ArXiv
\textbf{abs/2004.12983}
(2020)
\end{botherref}
\endbibitem

\bibitem[\protect\citeauthoryear{G'alvez et~al.}{2021}]{Galvez2021TighterEG}
\begin{bchapter}
\bauthor{\bsnm{G'alvez}, \binits{B.R.}},
\bauthor{\bsnm{Bassi}, \binits{G.}},
\bauthor{\bsnm{Thobaben}, \binits{R.}},
\bauthor{\bsnm{Skoglund}, \binits{M.}}:
\bctitle{Tighter expected generalization error bounds via wasserstein distance}.
In: \bbtitle{Neural Information Processing Systems}
(\byear{2021}).
\burl{https://api.semanticscholar.org/CorpusID:231698706}
\end{bchapter}
\endbibitem

\bibitem[\protect\citeauthoryear{Zhou et~al.}{2020}]{Zhou2020IndividuallyCI}
\begin{botherref}
\oauthor{\bsnm{Zhou}, \binits{R.}},
\oauthor{\bsnm{Tian}, \binits{C.}},
\oauthor{\bsnm{Liu}, \binits{T.}}:
Individually conditional individual mutual information bound on generalization error.
2021 IEEE International Symposium on Information Theory (ISIT),
670--675
(2020)
\end{botherref}
\endbibitem

\bibitem[\protect\citeauthoryear{Wang et~al.}{2021}]{Wang2021AnalyzingTG}
\begin{bchapter}
\bauthor{\bsnm{Wang}, \binits{H.}},
\bauthor{\bsnm{Huang}, \binits{Y.}},
\bauthor{\bsnm{Gao}, \binits{R.}},
\bauthor{\bsnm{Pin~Calmon}, \binits{F.}}:
\bctitle{Analyzing the generalization capability of sgld using properties of gaussian channels}.
In: \bbtitle{Neural Information Processing Systems}
(\byear{2021}).
\burl{https://api.semanticscholar.org/CorpusID:240354818}
\end{bchapter}
\endbibitem

\bibitem[\protect\citeauthoryear{Neu}{2021}]{Neu2021InformationTheoreticGB}
\begin{bchapter}
\bauthor{\bsnm{Neu}, \binits{G.}}:
\bctitle{Information-theoretic generalization bounds for stochastic gradient descent}.
In: \bbtitle{Annual Conference Computational Learning Theory}
(\byear{2021}).
\burl{https://api.semanticscholar.org/CorpusID:231740786}
\end{bchapter}
\endbibitem

\bibitem[\protect\citeauthoryear{Wang and Mao}{2021}]{Wang2021OnTG}
\begin{botherref}
\oauthor{\bsnm{Wang}, \binits{Z.}},
\oauthor{\bsnm{Mao}, \binits{Y.}}:
On the generalization of models trained with sgd: Information-theoretic bounds and implications.
ArXiv
\textbf{abs/2110.03128}
(2021)
\end{botherref}
\endbibitem

\bibitem[\protect\citeauthoryear{Shwartz-Ziv and Tishby}{2017}]{ShwartzZiv2017OpeningTB}
\begin{botherref}
\oauthor{\bsnm{Shwartz-Ziv}, \binits{R.}},
\oauthor{\bsnm{Tishby}, \binits{N.}}:
Opening the black box of deep neural networks via information.
ArXiv
\textbf{abs/1703.00810}
(2017)
\end{botherref}
\endbibitem

\bibitem[\protect\citeauthoryear{Ahuja et~al.}{2021}]{Ahuja2021InvariancePM}
\begin{bchapter}
\bauthor{\bsnm{Ahuja}, \binits{K.}},
\bauthor{\bsnm{Caballero}, \binits{E.}},
\bauthor{\bsnm{Zhang}, \binits{D.}},
\bauthor{\bsnm{Bengio}, \binits{Y.}},
\bauthor{\bsnm{Mitliagkas}, \binits{I.}},
\bauthor{\bsnm{Rish}, \binits{I.}}:
\bctitle{Invariance principle meets information bottleneck for out-of-distribution generalization}.
In: \bbtitle{Neural Information Processing Systems}
(\byear{2021}).
\burl{https://api.semanticscholar.org/CorpusID:235422593}
\end{bchapter}
\endbibitem

\bibitem[\protect\citeauthoryear{Wongso et~al.}{2023}]{Wongso2023UsingSM}
\begin{bchapter}
\bauthor{\bsnm{Wongso}, \binits{S.}},
\bauthor{\bsnm{Ghosh}, \binits{R.}},
\bauthor{\bsnm{Motani}, \binits{M.}}:
\bctitle{Using sliced mutual information to study memorization and generalization in deep neural networks}.
In: \bbtitle{International Conference on Artificial Intelligence and Statistics}
(\byear{2023}).
\burl{https://api.semanticscholar.org/CorpusID:259093297}
\end{bchapter}
\endbibitem

\bibitem[\protect\citeauthoryear{Hellstr{\"o}m and Durisi}{2022}]{Hellstrm2022EvaluatedCB}
\begin{botherref}
\oauthor{\bsnm{Hellstr{\"o}m}, \binits{F.}},
\oauthor{\bsnm{Durisi}, \binits{G.}}:
Evaluated cmi bounds for meta learning: Tightness and expressiveness.
ArXiv
\textbf{abs/2210.06511}
(2022)
\end{botherref}
\endbibitem

\bibitem[\protect\citeauthoryear{He et~al.}{2021}]{He2021InformationTheoreticCO}
\begin{barticle}
\bauthor{\bsnm{He}, \binits{H.}},
\bauthor{\bsnm{Yan}, \binits{H.}},
\bauthor{\bsnm{Tan}, \binits{V.Y.F.}}:
\batitle{Information-theoretic characterization of the generalization error for iterative semi-supervised learning}.
\bjtitle{J. Mach. Learn. Res.}
\bvolume{23},
\bfpage{287}--\blpage{128752}
(\byear{2021})
\end{barticle}
\endbibitem

\bibitem[\protect\citeauthoryear{Wu et~al.}{2020}]{Wu2020InformationtheoreticAF}
\begin{botherref}
\oauthor{\bsnm{Wu}, \binits{X.}},
\oauthor{\bsnm{Manton}, \binits{J.H.}},
\oauthor{\bsnm{Aickelin}, \binits{U.}},
\oauthor{\bsnm{Zhu}, \binits{J.}}:
Information-theoretic analysis for transfer learning.
2020 IEEE International Symposium on Information Theory (ISIT),
2819--2824
(2020)
\end{botherref}
\endbibitem

\bibitem[\protect\citeauthoryear{Hochreiter and Schmidhuber}{1997}]{Hochreiter1997FlatM}
\begin{barticle}
\bauthor{\bsnm{Hochreiter}, \binits{S.}},
\bauthor{\bsnm{Schmidhuber}, \binits{J.}}:
\batitle{Flat minima}.
\bjtitle{Neural Computation}
\bvolume{9},
\bfpage{1}--\blpage{42}
(\byear{1997})
\end{barticle}
\endbibitem

\bibitem[\protect\citeauthoryear{Foret et~al.}{2020}]{Foret2020SharpnessAwareMF}
\begin{botherref}
\oauthor{\bsnm{Foret}, \binits{P.}},
\oauthor{\bsnm{Kleiner}, \binits{A.}},
\oauthor{\bsnm{Mobahi}, \binits{H.}},
\oauthor{\bsnm{Neyshabur}, \binits{B.}}:
Sharpness-aware minimization for efficiently improving generalization.
ArXiv
\textbf{abs/2010.01412}
(2020)
\end{botherref}
\endbibitem

\bibitem[\protect\citeauthoryear{Jastrzebski et~al.}{2017}]{Jastrzebski2017ThreeFI}
\begin{botherref}
\oauthor{\bsnm{Jastrzebski}, \binits{S.}},
\oauthor{\bsnm{Kenton}, \binits{Z.}},
\oauthor{\bsnm{Arpit}, \binits{D.}},
\oauthor{\bsnm{Ballas}, \binits{N.}},
\oauthor{\bsnm{Fischer}, \binits{A.}},
\oauthor{\bsnm{Bengio}, \binits{Y.}},
\oauthor{\bsnm{Storkey}, \binits{A.J.}}:
Three factors influencing minima in sgd.
ArXiv
\textbf{abs/1711.04623}
(2017)
\end{botherref}
\endbibitem

\bibitem[\protect\citeauthoryear{Lewkowycz et~al.}{2020}]{Lewkowycz2020TheLL}
\begin{botherref}
\oauthor{\bsnm{Lewkowycz}, \binits{A.}},
\oauthor{\bsnm{Bahri}, \binits{Y.}},
\oauthor{\bsnm{Dyer}, \binits{E.}},
\oauthor{\bsnm{Sohl-Dickstein}, \binits{J.N.}},
\oauthor{\bsnm{Gur-Ari}, \binits{G.}}:
The large learning rate phase of deep learning: the catapult mechanism.
ArXiv
\textbf{abs/2003.02218}
(2020)
\end{botherref}
\endbibitem

\bibitem[\protect\citeauthoryear{Dinh et~al.}{2017}]{Dinh2017SharpMC}
\begin{bchapter}
\bauthor{\bsnm{Dinh}, \binits{L.}},
\bauthor{\bsnm{Pascanu}, \binits{R.}},
\bauthor{\bsnm{Bengio}, \binits{S.}},
\bauthor{\bsnm{Bengio}, \binits{Y.}}:
\bctitle{Sharp minima can generalize for deep nets}.
In: \bbtitle{International Conference on Machine Learning}
(\byear{2017}).
\burl{https://api.semanticscholar.org/CorpusID:7636159}
\end{bchapter}
\endbibitem

\bibitem[\protect\citeauthoryear{Du et~al.}{2018}]{Du2018GradientDP}
\begin{botherref}
\oauthor{\bsnm{Du}, \binits{S.S.}},
\oauthor{\bsnm{Zhai}, \binits{X.}},
\oauthor{\bsnm{P{\'o}czos}, \binits{B.}},
\oauthor{\bsnm{Singh}, \binits{A.}}:
Gradient descent provably optimizes over-parameterized neural networks.
ArXiv
\textbf{abs/1810.02054}
(2018)
\end{botherref}
\endbibitem

\bibitem[\protect\citeauthoryear{Zou et~al.}{2018}]{Zou2018StochasticGD}
\begin{botherref}
\oauthor{\bsnm{Zou}, \binits{D.}},
\oauthor{\bsnm{Cao}, \binits{Y.}},
\oauthor{\bsnm{Zhou}, \binits{D.}},
\oauthor{\bsnm{Gu}, \binits{Q.}}:
Stochastic gradient descent optimizes over-parameterized deep relu networks.
ArXiv
\textbf{abs/1811.08888}
(2018)
\end{botherref}
\endbibitem

\bibitem[\protect\citeauthoryear{Allen-Zhu et~al.}{2018}]{AllenZhu2018ACT}
\begin{botherref}
\oauthor{\bsnm{Allen-Zhu}, \binits{Z.}},
\oauthor{\bsnm{Li}, \binits{Y.}},
\oauthor{\bsnm{Song}, \binits{Z.}}:
A convergence theory for deep learning via over-parameterization.
ArXiv
\textbf{abs/1811.03962}
(2018)
\end{botherref}
\endbibitem

\bibitem[\protect\citeauthoryear{Sirignano and Spiliopoulos}{2018}]{Sirignano2018MeanFA}
\begin{botherref}
\oauthor{\bsnm{Sirignano}, \binits{J.A.}},
\oauthor{\bsnm{Spiliopoulos}, \binits{K.V.}}:
Mean field analysis of neural networks: A central limit theorem.
Stochastic Processes and their Applications
(2018)
\end{botherref}
\endbibitem

\bibitem[\protect\citeauthoryear{Mei et~al.}{2018}]{Mei2018AMF}
\begin{barticle}
\bauthor{\bsnm{Mei}, \binits{S.}},
\bauthor{\bsnm{Montanari}, \binits{A.}},
\bauthor{\bsnm{Nguyen}, \binits{P.-M.}}:
\batitle{A mean field view of the landscape of two-layer neural networks}.
\bjtitle{Proceedings of the National Academy of Sciences of the United States of America}
\bvolume{115},
\bfpage{7665}--\blpage{7671}
(\byear{2018})
\end{barticle}
\endbibitem

\bibitem[\protect\citeauthoryear{Chizat and Bach}{2018}]{Chizat2018OnTG}
\begin{botherref}
\oauthor{\bsnm{Chizat}, \binits{L.}},
\oauthor{\bsnm{Bach}, \binits{F.R.}}:
On the global convergence of gradient descent for over-parameterized models using optimal transport.
ArXiv
\textbf{abs/1805.09545}
(2018)
\end{botherref}
\endbibitem

\bibitem[\protect\citeauthoryear{Chizat et~al.}{2018}]{Chizat2018OnLT}
\begin{bchapter}
\bauthor{\bsnm{Chizat}, \binits{L.}},
\bauthor{\bsnm{Oyallon}, \binits{E.}},
\bauthor{\bsnm{Bach}, \binits{F.R.}}:
\bctitle{On lazy training in differentiable programming}.
In: \bbtitle{Neural Information Processing Systems}
(\byear{2018}).
\burl{https://api.semanticscholar.org/CorpusID:189928159}
\end{bchapter}
\endbibitem

\bibitem[\protect\citeauthoryear{Arjevani and Field}{2022}]{Arjevani2022AnnihilationOS}
\begin{botherref}
\oauthor{\bsnm{Arjevani}, \binits{Y.}},
\oauthor{\bsnm{Field}, \binits{M.}}:
Annihilation of spurious minima in two-layer relu networks.
ArXiv
\textbf{abs/2210.06088}
(2022)
\end{botherref}
\endbibitem

\bibitem[\protect\citeauthoryear{Yun et~al.}{2018}]{Yun2018SmallNI}
\begin{bchapter}
\bauthor{\bsnm{Yun}, \binits{C.}},
\bauthor{\bsnm{Sra}, \binits{S.}},
\bauthor{\bsnm{Jadbabaie}, \binits{A.}}:
\bctitle{Small nonlinearities in activation functions create bad local minima in neural networks}.
In: \bbtitle{International Conference on Learning Representations}
(\byear{2018}).
\burl{https://api.semanticscholar.org/CorpusID:52893515}
\end{bchapter}
\endbibitem

\bibitem[\protect\citeauthoryear{Jacot et~al.}{2018}]{Jacot2018NeuralTK}
\begin{botherref}
\oauthor{\bsnm{Jacot}, \binits{A.}},
\oauthor{\bsnm{Gabriel}, \binits{F.}},
\oauthor{\bsnm{Hongler}, \binits{C.}}:
Neural tangent kernel: Convergence and generalization in neural networks.
ArXiv
\textbf{abs/1806.07572}
(2018)
\end{botherref}
\endbibitem

\bibitem[\protect\citeauthoryear{Csiszar}{1967}]{1967Information}
\begin{botherref}
\oauthor{\bsnm{Csiszar}, \binits{I.}}:
Information-type measures of difference of probability distributions and indirect observations.
Studia . Math. Hungar
\textbf{2}
(1967)
\end{botherref}
\endbibitem

\bibitem[\protect\citeauthoryear{Kullback}{1967}]{Kullback1967ALB}
\begin{barticle}
\bauthor{\bsnm{Kullback}, \binits{S.}}:
\batitle{A lower bound for discrimination information in terms of variation (corresp.)}.
\bjtitle{IEEE Trans. Inf. Theory}
\bvolume{13},
\bfpage{126}--\blpage{127}
(\byear{1967})
\end{barticle}
\endbibitem

\bibitem[\protect\citeauthoryear{Cover et~al.}{2003}]{2003Elements}
\begin{botherref}
\oauthor{\bsnm{Cover}, \binits{T.}},
\oauthor{\bsnm{Thomas}, \binits{J.}},
\oauthor{\bsnm{Wiley}, \binits{J.}}:
Elements of information theory.
Tsinghua University Pres
(2003)
\end{botherref}
\endbibitem

\bibitem[\protect\citeauthoryear{Yang et~al.}{2019}]{Yang2019CARSCE}
\begin{botherref}
\oauthor{\bsnm{Yang}, \binits{Z.}},
\oauthor{\bsnm{Wang}, \binits{Y.}},
\oauthor{\bsnm{Chen}, \binits{X.}},
\oauthor{\bsnm{Shi}, \binits{B.}},
\oauthor{\bsnm{Xu}, \binits{C.}},
\oauthor{\bsnm{Xu}, \binits{C.}},
\oauthor{\bsnm{Tian}, \binits{Q.}},
\oauthor{\bsnm{Xu}, \binits{C.}}:
Cars: Continuous evolution for efficient neural architecture search.
2020 IEEE/CVF Conference on Computer Vision and Pattern Recognition (CVPR),
1826--1835
(2019)
\end{botherref}
\endbibitem

\end{thebibliography}

\end{document}